\documentclass{article}

\usepackage[preprint]{neurips_2025}
\usepackage{amsmath,amssymb,amsthm}

\usepackage[dvipsnames]{xcolor}
\usepackage[utf8]{inputenc} %
\usepackage[T1]{fontenc}    %
\usepackage{hyperref}       %
\usepackage{url}            %
\usepackage{booktabs}       %
\usepackage{amsfonts}       %
\usepackage{nicefrac}       %
\usepackage{microtype}      %
\usepackage{subcaption}

\usepackage{parskip}

\usepackage{natbib}
\usepackage{amsfonts,bm}

\usepackage{amsmath, amssymb,mathrsfs, amsthm, mathtools}
\usepackage{xspace}
\usepackage{enumitem}
\usepackage{lipsum}
\usepackage{xpatch}
\usepackage{mathtools}
\usepackage{dsfont}
\usepackage{pgf,tikz}
\usetikzlibrary{positioning,matrix,patterns}
\usepackage{caption}
\usepackage{graphicx}
\usepackage{makecell}
\usepackage{multirow}
\usepackage[mathscr]{euscript}
\definecolor{hanblue}{rgb}{0.27, 0.42, 0.81}
\definecolor{deepred}{HTML}{900C3F}
\definecolor{deepgreen}{HTML}{2F6960}
\hypersetup{
    colorlinks=true,
    linkcolor=hanblue,
    urlcolor=hanblue,
    citecolor=hanblue,
    anchorcolor=hanblue}

 \usepackage[capitalize,noabbrev]{cleveref}

\theoremstyle{plain}
\newtheorem{theorem}{Theorem}
\newtheorem{proposition}[theorem]{Proposition}

\newtheorem{lemma}[theorem]{Lemma}

\theoremstyle{definition}

\newtheorem{assumption}[theorem]{Assumption}
\theoremstyle{remark}

\def\calP{{\mathcal{P}}}

\def\sR{{\mathbb{R}}}

\DeclareMathOperator*{\argmin}{arg\,min}

\DeclarePairedDelimiter\abs{\lvert}{\rvert}%
\DeclarePairedDelimiter\norm{\lVert}{\rVert}%

\makeatletter
\let\oldabs\abs
\def\abs{\@ifstar{\oldabs}{\oldabs*}}
\let\oldnorm\norm
\def\norm{\@ifstar{\oldnorm}{\oldnorm*}}
\makeatother

\newcommand{\eps}{\varepsilon}
\newcommand{\dd}{\mathrm{d}}

\def\E{{\mathbb{E}}}

\def\Prob{{\mathbb{P}}}
\def\R{{\mathbb{R}}}

\newcommand*{\defeq}{\coloneqq}

\newcommand{\bff}{\mathbf{f}}
\newcommand{\ba}{\mathbf{a}}
\newcommand{\bx}{\mathbf{x}}
\newcommand{\bX}{\mathbf{X}}
\newcommand{\bY}{\mathbf{Y}}
\newcommand{\bg}{\mathbf{g}}
\newcommand{\bb}{\mathbf{b}}
\newcommand{\by}{\mathbf{y}}

\DeclareMathOperator{\Var}{Var}

\usepackage{cleveref,wrapfig}
\usepackage{algpseudocode}
\usepackage{algorithm}
\renewcommand{\bm}[1]{\mathbf{#1}}
\usepackage{sistyle}
\SIthousandsep{,}

\title{\resizebox{\textwidth}{!}{On {\kern-.05em}Fitting {\kern-.05em}Flow {\kern-.05em}Models {\kern-.05em}with {\kern-.05em}Large {\kern-.05em}Sinkhorn {\kern-.05em}Couplings}}

\author{
  Stephen Zhang\thanks{Equal contributions. Work done while doing an internship at Apple.}\,\;{\normalfont\textsuperscript{1,3}} \quad Alireza Mousavi-Hosseini\footnotemark[1]\,\;{\normalfont\textsuperscript{1,2}} \quad Michal Klein{\normalfont\textsuperscript{1}} \quad  Marco Cuturi{\normalfont\textsuperscript{1}}\\
  \textsuperscript{1}Apple \quad \textsuperscript{2}University of Toronto \quad \textsuperscript{3}University of Melbourne\\
  \texttt{michalk,stephen\_zhang5,cuturi@apple.com}\\
  \texttt{mousavi@cs.toronto.edu}\\
}

\begin{document}

\maketitle

\begin{abstract}
Flow models transform data gradually from one modality (e.g. noise) onto another (e.g. images).
Such models are parameterized by a time-dependent velocity field, trained to fit segments connecting pairs of source and target points. 
When the pairing between source and target points is given, training flow models boils down to a supervised regression problem. When no such pairing exists, as is the case when generating data from noise, training flows is much harder.
A popular approach lies in picking source and target points independently~\citep{lipman2023flow}. This can, however, lead to velocity fields that are slow to train, but also costly to integrate at inference time.
\textit{In theory}, one would greatly benefit from training flow models by sampling pairs from an optimal transport (OT) measure coupling source and target, since this would lead to a highly efficient flow solving the \citeauthor{benamou2000computational} dynamical OT problem. %
\textit{In practice}, recent works have proposed to sample \textit{mini-batches} of $n$ source and $n$ target points and reorder them using an OT solver to form \textit{better} pairs. These works have advocated using batches of size $n\approx 256$, and considered OT solvers that return couplings that are either sharp (using e.g. the Hungarian algorithm) or blurred (using e.g. entropic regularization, a.k.a. \citeauthor{Sinkhorn64}).
We follow in the footsteps of these works by exploring the benefits of increasing this mini-batch size $n$ by three to four orders of magnitude, and look more carefully on the effect of the entropic regularization $\varepsilon$ used in the \citeauthor{Sinkhorn64} algorithm. Our analysis is facilitated by new scale invariant quantities to report the sharpness of a coupling, while our sharded computations across multiple GPU or GPU nodes allow scaling up $n$.
We show that in both synthetic and image generation tasks, flow models greatly benefit when fitted with large \citeauthor{Sinkhorn64} couplings, with a low entropic regularization $\varepsilon$.
\end{abstract}

\section{Introduction}
Finding a map that can transform a source into a target measure is a task at the core of generative modeling and unpaired modality translation.
Following the widespread popularity of GAN formulations~\citep{goodfellow2014generative}, the field has greatly benefited from a gradual, time-dependent parameterization of these transformations as normalizing flows~\citep{rezende2015variational} and neural ODEs~\citep{chen2018neural}. Such flow models are now commonly estimated using flow matching~\citep{lipman2024flow}.
While a velocity formulation substantially increases the expressivity of generative models, this results on the other hand in a higher cost at inference time due to the additional burden of running an ODE solver.
Indeed, a common drawback of Neural-ODE solvers is that that they require potentially many steps, and therefore many passes through the flow network, to generate data.
In principle, to mitigate this problem, the gold standard for such continuous-time transformations is given by the solution of the~\citeauthor{benamou2000computational} dynamical optimal transport (OT) problem, which should be equivalent, if trained perfectly, to a 1-step generation achieved by the~\citeauthor{Monge1781} map formulation~\citep[\S1.3]{santambrogio2015optimal}.
In practice, while the mathematics~\citep{villani2003topics} of optimal transport have contributed to the understanding of these methods~\citep{liuflow}, the jury seems to be still out on ruling whether tools from the computational OT toolbox~\citep{Peyre2019computational}, which is typically used to compute large scale couplings on data~\citep{klein2025mapping}, can decisively help with the estimation of flows in high-dimensional / high-sample sizes regimes.

\paragraph{Stochastic interpolants.} The flow matching (FM) framework \citep{lipman2024flow}, introduced in concurrent seminal papers \citep{peluchetti2022nondenoising,lipman2023flow,albergo2023building,neklyudov2023action} proposes to estimate a flow model by leveraging a pre-defined interpolation $\mu_t$ between source $\mu_0$ and target $\mu_1$ measures --- the stochastic interpolant following the terminology of~\citeauthor{albergo2023building}. That interpolation is the crucial ingredient used to fit a parameterized velocity field with a regression loss. In practice, such an interpolation can be formed by sampling $X_0\sim \mu_0$ independently of $X_1\sim \mu_1$ and defining $\mu_t$ as the law of $X_t:= (1-t) X_0 + t X_1$. One can then fit a parameterized time-dependent velocity field $\mathbf{v}_\theta(t, \bx)$ that minimizes the expectation of $\| X_1 - X_0 -\mathbf{v}_\theta(X_T, T)\|^2$ w.r.t. $X_0,X_1$ and $T$ a random time variable in $[0,1]$.
This procedure (hereafter abbreviated as Independent-FM, I-FM) has been immensely successful, but can suffer from high variance, and as highlighted by~\citep{liu2022rectified} the I-FM loss can never be zero. Furthermore, minimizing it cannot recover an optimal transport path: the effect of this can be measured by noticing a high curvature when integrating the ODE needed to form an output from an input sample point $\bx_0$.
\paragraph{From I-FM to Batch-OT FM.}
To fit exactly the OT framework, ideally one would choose $\mu_t$ to be the \citeauthor{mccann1997convexity} interpolation between $\mu_0$ and $\mu_1$, which would be $\mu_t:=((1-t)\mathrm{Id}+tT^\star)_\#\mu_0$, where $T^\star$ is the \citeauthor{Monge1781} map connecting $\mu_0$ to $\mu_1$. Unfortunately, this insight is irrelevant, since knowing $T^\star$ would mean that no flow needs to be trained at all. Adopting a more practical perspective, ~\citet{pooladian2023multisample} and~\citet{tong2023simulation} have proposed to modify I-FM and select pairs of source and target points more carefully, using discrete OT solvers. Concretely, they sample mini-batches $\mathbf{x}_0^1\dots, \mathbf{x}_0^n$ from $\mu_0$ and $\mathbf{x}_1^1,\dots, \mathbf{x}_1^n$ from $\mu_1$; compute an $n\times n$ OT coupling matrix; sample pairs of indices $(i_\ell,j_\ell)$ from that bistochastic matrix, and feed the flow model with pairs $\mathbf{x}_0^{i_\ell},\mathbf{x}_1^{j_\ell}$. This approach, referred to as Batch OT-FM in the literature, was recently used and adapted in \citet{tian2024equiflow,generale2024conditional,klein2023equivariant,davtyanfaster,kim2024simple}. 
Despite their appeal, these modifications have not yet been widely adopted. The consensus stated recently by~\citet{lipman2024flow} seems to be still that \textit{"the most popular class of affine probability paths is instantiated by the independent coupling"}. 

\paragraph{Can mini-batch OT really help?}
We try to answer this question by noticing first that the evaluations carried out in all of the references cited above use batch sizes of $2^8=256$ points, more rarely $2^{10}=1024$, upper bounded by $2^{12}=4096$ for \cite{kim2024simple}. We believe that for many of these works this might be due to a reliance on the Hungarian algorithm~\citep{kuhn1955hungarian} whose $O(n^3)$ complexity is prohibitive for large $n$. We also notice that, while these works also consider entropic OT (EOT)~\citep{cuturi2013sinkhorn}, they stick to a single $\varepsilon$ regularization value in their evaluations (e.g. 0.2 \cite{kim2024simple}). We go back to the drawing board in this paper, and study whether batch OT-FM can reliably work, and if so at which regimes of mini-batch size $n$, regularization $\varepsilon$, and for which data dimensions $d$. Our contributions are:
\begin{itemize}[leftmargin=.3cm,itemsep=.05cm,topsep=0cm,parsep=2pt]
\item Rather than drawing an artificial line between Batch-OT (in Hungarian or EOT form) and I-FM, we leverage the fact that \textit{all} of these approaches can be interpolated using EOT: Hungarian corresponds to the case where $\varepsilon\rightarrow 0$ while I-FM is recovered with $\varepsilon\rightarrow \infty$. I-FM is therefore a particular case of Batch-OT with infinite regularization, which can be continuously moved towards batch-OT.
\item We modify the \citeauthor{Sinkhorn64} algorithm when used with the squared-Euclidean cost: we drop norms and only use negative dot-product. This improves stability and still returns the correct solution.
\item We define a renormalized entropy for couplings, to pin them efficiently on a scale of 0 (bijective assignment induced by a permutation, e.g. that returned by the Hungarian algorithm) to 1 (independent coupling). This quantity is useful because, unlike transport cost or entropy regularization $\varepsilon$, it is bounded in $[0,1]$ and does not depend on the data dimension $d$ or coupling size $n\times n$.
\item We explore in our experiments substantially different regimes for $n$ and $\varepsilon$. We vary the mini-batch size from $n=2^{11}=\num{2048}$ to $n=2^{21}=\num{2097152}$ and consider an adaptive grid to set $\varepsilon$ that results in \citeauthor{Sinkhorn64} couplings whose normalized entropy is distributed within $[0,1]$.
\end{itemize} 

\section{Background Material on Optimal Transport and Flow Matching}\label{sec:background}

Let $\calP_2(\R^d)$ denote the space of probability measures over $\R^d$ with finite second moment. Let $\mu, \nu \in \calP_2(\R^d)$, and let $\Gamma(\mu,\nu)$ be the set of joint probability measures in $\calP_2(\R^d\times\R^d)$ with left-marginal $\mu$ and right-marginal $\nu$. The OT problem in its~\citeauthor{Kantorovich42} formulation is:
\begin{align}\label{eq:wassdist}
    W_2(\mu,\nu)^2 \defeq \inf_{\pi \in \Gamma(\mu,\nu)}\iint \frac12 \|x-y\|^2\dd\pi(x,y)\,.
\end{align}
A minimizer of \eqref{eq:wassdist} is called an \emph{OT coupling measure}, denoted $\pi^\star$. If $\mu$ was a noise source and $\nu$ a data target measure, $\pi^\star$ would be the perfect coupling to sample pairs of noise and data to learn flow models: sample $\bx_0, \bx_1\sim \pi^\star$ and ensure the flow models bring $\bx_0$ to $\bx_1$. Such optimal couplings $\pi^\star$ are in fact induced by \emph{pushforward maps}: when paired optimally, a point $\bx_0$ can only be associated with a $\bx_1=T(\bx_0)$, where $T:\R^d\rightarrow\R^d$ is the \citeauthor{Monge1781} optimal transport map, defined as follows:
\begin{align}\label{eq:monge}
    T^\star(\mu,\nu) \coloneqq \argmin_{T: T_{\#}\mu=\nu}\int \frac12 \|\bx - T(\bx)\|^2\dd\mu(\bx)\,
\end{align}
 where the push-forward constraint $T_\# \mu=\nu$ means that for $X \sim \mu$ one has $T(X)\sim\nu$. \citeauthor{Monge1781} OT maps have been characterized by~\citeauthor{Bre91} in great detail:

\begin{theorem}[\citep{Bre91}]\label{thm:brenier_thm} If $\mu \in \calP_{2}(\R^d)$ has an absolutely continuous density then \eqref{eq:monge} is solved by a map $T^\star$ of the form $T^\star=\nabla u$, where $u:\R^d\rightarrow \R$ is convex. Moreover if $u$ is a convex potential that is such that $\nabla u_\#\mu=\nu$ then $\nabla u$ solves \eqref{eq:monge}.
\end{theorem}

As a result of Theorem~\ref{thm:brenier_thm}, one can choose an arbitrary convex potential $u$, a starting measure $\mu$, and define a synthetic task to train flow matching models between $\mu_0:=\mu$ and $\mu_1:=\nabla u_{\#}\mu$, for which a ground truth coupling $\pi^\star$ is known. Inspired by~\citep{korotin2021neural} who considered the same result to benchmark \citet{Monge1781} map solvers, we use this setting in \S~\ref{subsec:exp_synth1} to benchmark batch-OT.

\paragraph{Entropic OT.}
Entropic regularization~\citep{cuturi2013sinkhorn} has become the most popular approach to estimate a finite sample analog of $\pi^\star$ using samples $(\bx_1,\ldots,\bx_n)$ and $(\by_1,\ldots,\by_n)$. 
Using a regularization strength $\varepsilon>0$, a cost matrix $\mathbf{C}:=[\tfrac12\|\bx_i - \by_j\|^2]_{ij}$ between these samples, the entropic OT (EOT) problem can be presented in primal and dual forms as:
\begin{equation}\label{eq:entdual}
    \min_{\mathbf{P}\in\R^{n\times n}_+,\mathbf{P}\mathbf{1}_n=\mathbf{P}^T\mathbf{1}_n=\mathbf{1}_n/n}\!\!\!\!\!\! \langle \mathbf{P},\mathbf{C} \rangle -\varepsilon H(\mathbf{P}), \quad
    \max_{\substack{ \bff, \bg \in \R^n }} \tfrac1n\langle \bff+\bg,\mathbf{1}_n \rangle - \eps \langle \exp\left(\tfrac{\bff\oplus\bg - \mathbf{C}}{\eps}\right),\mathbf{1}_{n\times n}\rangle,
\end{equation}
where $H(\mathbf{P}) = -\langle \mathbf{P}, \log(\mathbf{P}) \rangle$ is the discrete entropy functional.

\begin{wrapfigure}{r}{0.5\textwidth}
\begin{minipage}{0.5\textwidth}
\vspace{-15pt}
\begin{algorithm}[H]
\caption{\textsc{Sink}$(\bX\in\R^{n\times d}, \bY\in\R^{n\times d}, \varepsilon, \tau)$}
\label{algo:sinkhorn}
\begin{algorithmic}[1]
    \State{$\bff, \bg\leftarrow \mathbf{0}_n, \mathbf{0}_n$.}
    \State{$\mathbf{C}\leftarrow [\tfrac12\|\bx_i-\by_j\|^2]_{ij}, i\leq n, j\leq n$}\label{line:sink_change}
    \While{$\|\exp\left(\tfrac{\bff\oplus\bg - \mathbf{C}}{\eps}\right)\mathbf{1}_n-\tfrac1n\mathbf{1}_n\|_1 > \tau$}
        \State $\bff \leftarrow  \varepsilon\log \tfrac1n\mathbf{1}_n + \min_\varepsilon(\mathbf{C}-\bff\oplus\bg) + \bff$\label{line:minf}
        \State $\bg \leftarrow \varepsilon\log \tfrac1n\mathbf{1}_n + \min_\varepsilon(\mathbf{C}^\top-\bg\oplus\bff) + \bg$\label{line:ming}
    \EndWhile
    \State{{\bfseries return} {$\bff, \bg, \mathbf{P}=\exp\left((\bff\oplus\bg - \mathbf{C})/\eps\right)$ \label{lst:line:coupling}}}
\end{algorithmic}
\end{algorithm}
\vspace{-20pt}
\end{minipage}
\end{wrapfigure}

The optimal solutions to \eqref{eq:entdual} are usually found with the \citeauthor{Sinkhorn64} algorithm, as presented in Algorithm~\ref{algo:sinkhorn}, where for a matrix  
$\mathbf{S}$ we write $\min_\varepsilon(\mathbf{S}) \coloneqq [-\varepsilon \log\left( \mathbf{1}^\top e^{-\mathbf{S}_{i\cdot}/\varepsilon}\right)]_i$, and $\oplus$ is the tensor sum of two vectors, i.e. $(\bff\oplus \bg)_{ij} := \bff_i+\bg_j.$
The optimal dual variables \eqref{eq:entdual} $(\bff^\varepsilon,\bg^\varepsilon)$ can then be used to instantiate a valid coupling matrix $\mathbf{P}^\varepsilon = \exp\left((\bff^\varepsilon\oplus\bg^\varepsilon - \mathbf{C})/\eps\right)$, which approximately solves the finite-sample counterpart of \eqref{eq:wassdist}. An important remark is that as $\varepsilon\rightarrow0$, the solution $\mathbf{P}^\varepsilon$ converges to the optimal transport matrix solving \eqref{eq:wassdist}, while $\mathbf{P}^\varepsilon\rightarrow \tfrac{1}{n^2}\mathbf{1}_{n\times n}$ as $\varepsilon\rightarrow \infty$. These two limiting points coincide with the \textit{optimal assignment} matrix (or optimal permutation as returned e.g. by the Hungarian algorithm~\citep{kuhn1955hungarian}), and the uniform independent coupling used implicitly in I-FM.

\paragraph{Independent and Batch-OT Flow Matching.}
\begin{wrapfigure}{r}{0.48\textwidth}
\begin{minipage}{0.48\textwidth}
\vspace{-25pt}
\begin{algorithm}[H]
    \caption{\textsc{FM} 1-Step\label{alg:train}\smash{$(\mu_0,\mu_1,n,\textsc{OT-Solve})$}}
    \begin{algorithmic}[1]
        \State \label{step:1} $\bX_0 = (\bx^1_0,
        \hdots,\bx^n_0) \sim \mu_0$\label{line:x0}
        \State \label{step:2} $\bX_1 = (\bx^1_1,\hdots,\bx^n_1) \sim \mu_1$\label{line:x1}
        \State \label{step:3} $\mathbf{P} \leftarrow \textsc{OT-Solve}(\bX_0,\bX_1) \text{ or } \mathbf{I}_n/n$
        \State \label{step:4} $(i_1,j_1), \dots, (i_n,j_n) \sim \mathbf{P}$\label{line:sampling}
        \State \label{step:5} $t_1, \dots, t_n \leftarrow \textsc{TimeSampler}$
        \State \label{step:6} $\tilde{\bx}^k \leftarrow (1-t_k)\bx^{i_k}_0 + t_k \bx^{j_k}_1, \text{ for } k\leq n$ 
        \State \label{step:7} $\mathcal{L}(\theta) = \sum_{k}\|\bx^{j_k}_1 - \bx^{i_k}_0 - \mathbf{v}_\theta(\tilde{\bx}^k,t_k)\|^2$
        \State \label{step:8} $\theta \leftarrow \textsc{Gradient-Update}(\nabla\mathcal{L}(\theta))$
    \end{algorithmic}
\end{algorithm}
\end{minipage}
\vspace{-6pt}
\end{wrapfigure}
FM methods use a stochastic interpolant $\mu_t$ with law $X_t:= (1-t) X_0 + t X_1$, to minimize the expectation of a squared-norm regression loss $\min_\theta \mathbb{E}_{T,X_0,X_1}\|X_1 - X_0 -\mathbf{v}_\theta(X_T, T)\|^2$ where 
$X_0\sim \mu_0,X_1\sim\mu_1$ and $T$ a random variable in $[0,1]$. In I-FM, this interpolant is implemented by taking independent batches of samples $\bx_0^1\dots, \bx_0^n$ from $\mu_0$, $\bx_1^1,\dots, \bx_1^n$ from $\mu_1$, and $t_1,\dots,t_n$ time values sampled in $[0,1]$, to form the loss values  $\|\bx_1^k-\bx_0^k-\mathbf{v}_\theta((1-t_k) \bx_0^k + t_k \bx_1^k, t_k)\|^2$. 
In the formalism of \citet{pooladian2023multisample} and \citet{tong2023simulation}, the same samples $\bx_0^1\dots, \bx_0^n$ and $\bx_1^1,\dots, \bx_1^n$ are first fed into a discrete optimal matching solver. 
This outputs a bistochastic coupling matrix $\mathbf{P}\in\mathbb{R}^{n\times n}$ which is then used to \textit{re-shuffle} the $n$ pairs originally provided to be better coupled, and which should help the velocity field fit straighter trajectories, with less training steps. 
The procedure is summarized in Algorithm~\ref{alg:train} and adapted to our setup and notations. The choice $\mathbf{I}_n/n$ corresponds to I-FM, as it would return the original untouched pairs $(\bx_0^k, \bx_1^k)$.
Equivalently, I-FM would also be recovered if the coupling was the independent coupling $\mathbf{1}_{n\times n}/n^2$, up to the difference on carrying out stratified sampling (which would result in each noise/image observed once per mini-batch) or sampling with replacement. More recently, \citet{davtyanfaster} have proposed to keep a memory of that matching effort across mini-batches, by updating a large (of the size of the entire dataset) assignment permutation between noise and full-batch data that is locally refreshed with the output of the Hungarian method run on a small batch. 

\textbf{Batch-OT as an Enhanced Dataloader} A crucial aspect of the batch-OT methodology is that in its current implementations, any effort done to pair data more carefully with noise is disconnected from the training of $\mathbf{v}_\theta$ itself. Indeed, as currently implemented, OT variants of FM can be interpreted as meta-dataloaders that do a selective pairing of noise and data, without considering $\theta$ at all in that pairing. In that sense, training and preparation of coupled noise/data pairs can be done independently.

\section{Prepping Sinkhorn for Large Batch Size and Dimension.}\label{sec:newsink}

\textbf{On Using Large Batch Size and Selecting $\varepsilon>0$.} The motivation to use larger batch sizes for Batch-OT lies in the fundamental bias introduced by using small batches in light of the OT curse of dimensionality~\citep{chewi2024statistical,fatras2019learning}, which cannot be traded off with more iterations on the flow matching loss. Specifically, we provide the following lower bound that characterizes the statistical hardness of optimal transport, and defer its proof to the \Cref{app:proof}.
\begin{proposition}\label{prop:var_lower_bound_informal}
    Suppose the support of $\mu_1$ has intrinsic dimension $r$, formalized in Assumption \ref{assump:data}. Define the coupling $X_0,X_1 \sim \pi_n$ as follows: first draw $\bX_0 \sim \mu_0^{\otimes n}$ and $\bX_1 \sim \mu_1^{\otimes n}$, then sample $X_0,X_1 \sim \hat{\pi}_n(\bX_0,\bX_1)$ for any coupling rule $\hat{\pi}_n$ supported on $\bX_0,\bX_1$. Then, for any $\bx_0 \in \sR^d$,
    \[
    \Var_{X_0,X_1 \sim \pi_n}(X_1 \,|\, X_0 = \bx_0) \geq cn^{-2/r},
    \]
    where $c > 0$ is a constant depending only on $C$ and $r$ of Assumption \ref{assump:data}.
\end{proposition}
Note that the above proposition covers the case of using couplings that are supported on batches of noise and data, as in \Cref{algo:sinkhorn}. When $\mu_0$ admits a density, the conditional variance under exact OT would be zero. Thus, \Cref{prop:var_lower_bound_informal} shows the curse of dimensionality in learning optimal transport \emph{with any high-dimensional data distribution} $\mu_1$, which is in contrast to minimax lower bounds (e.g.\ \citet[Theorem 2.15]{chewi2024statistical}) that only show the hardness for \emph{some} unknown pair of distributions. This generality is at the expense of limiting the (stochastic) coupling to be supported on $(\bX_0,\bX_1)$, which is the relevant setting for flow matching. This curse of dimensionality becomes milder under the \emph{manifold hypothesis} where $r \ll d$, but still advocates for the use of large $n$.

The necessity of varying $\varepsilon$ is that this regularization can offset the bias between a regularized empirical OT matrix and its coupling measure counterpart, with favorable sample complexity~\citep{genevay2018sample,mena2019statistical,rigollet2025sample}.

\textbf{Automatic Rescaling of $\varepsilon$.}
A practical problem arising when running the \citeauthor{Sinkhorn64} algorithm lies in choosing the $\varepsilon$ parameter. As described earlier, while $\mathbf{P}^\varepsilon$ does follow a path from the optimal permutation (i.e., returned by the Hungarian algorithm) to the independent coupling, as $\varepsilon$ varies from 0 to $\infty$, what matters in practice is to pick relevant values in between these two extremes. To avoid using a fixed grid that risks becoming irrelevant as we vary $n$ and $d$, we revisit the strategy originally used in~\citep{cuturi2013sinkhorn} to divide the cost matrix $\mathbf{C}$ by its mean, median or maximal value, as implemented for instance in~\citep{flamary2021pot}. While needed to avoid underflow when instantiating a kernel matrix $\mathbf{K}=e^{-\mathbf{C}/\varepsilon}$, that strategy is not relevant when using the \textrm{log-sum-exp} operator in our implementation (as advocated in~\citep[Remark 4.23]{Peyre2019computational}), since the $\min_\varepsilon$ in our implementation is \textit{invariant} to a constant shift in $\mathbf{C}$, whereas mean, median and max statistics are not. We propose instead to use the \textit{standard deviation} (STD) of the cost matrix. Indeed, the dispersion of costs around their mean has more relevance as a scale than the mean of these costs itself. The STD can be computed in $nd^2$ time / memory, without having to instantiate the cost matrix. When this memory cost increase from $d$ to $d^2$ is too high, we subsample $n=2^{14}=\num{16384}$ points. In what follows, we always pass the $\varepsilon$ value to the \citeauthor{Sinkhorn64} algorithm \ref{algo:sinkhorn} as $\tilde{\varepsilon}:= \textrm{std}(\mathbf{C})\times \varepsilon$, where $\varepsilon$ is now a scale-free quantity selected in a logarithmic grid within $[0.001,1.0]$.

\textbf{Scale-Free Renormalized Coupling Entropy.}
While useful to keep computations stable across runs, the rescaling of $\varepsilon$ still does not provide a clear idea of whether a computed coupling $\mathbf{P}^\varepsilon$ from $n$ to $n$ points is sharp (close to an optimal permutation) or blurred (closer to what I-FM would use). While a distance to the independent coupling can be computed, that to the optimal Hungarian permutation cannot, of course, be derived without computing it beforehand which would incur a prohibitive cost. Instead, we resort to a fundamental information inequality used in~\citep{cuturi2013sinkhorn}: if $\mathbf{P}$ is a valid coupling between two marginal probability vectors $\ba,\bb$, then one has 
$\tfrac12(H(\ba)+H(\bb)) \leq H(\mathbf{P})\leq H(\ba)+H(\bb).$
As a result, for any $\varepsilon$, we define the \textit{renormalized} entropy $\mathcal{E}$ of a coupling of $\mathbf{a},\mathbf{b}$:
$$\mathcal{E}(\mathbf{P}):= \frac{2H(\mathbf{P})}{H(\ba)+H(\bb)} -1 \in (0,1].$$
When $\ba=\bb=\mathbf{1}_n/n$, as considered in this work, this simplifies to $\mathcal{E}(\mathbf{P}):=H(\mathbf{P})/\log n -1$. Independently of the size $n$ and $\varepsilon$, $\mathcal{E}(\mathbf{P}^\varepsilon)$ provides a simple measure of the proximity of $\mathbf{P}^\varepsilon$ to an optimal assignment matrix (as $\mathcal{E}$ gets closer to 0) or to the independent coupling (as $\mathcal{E}$ reaches 1). As a result we report $\mathcal{E}(\mathbf{P}^\varepsilon)$ rather than $\varepsilon$ in our plots (or to be more accurate, the \textit{average} of $\mathcal{E}(\mathbf{P}^\varepsilon)$ computed over multiple mini-batches). Figures \ref{fig:piecewise-final-app-entropy} and \ref{fig:korotin-final-app-epsilon} in the appendix are indexed by $\varepsilon$ instead.

\textbf{From Squared Euclidean Costs to Dot-products.}
Using the notation $T^\star(\mu,\nu)$ introduced in \eqref{eq:monge}, we notice an equivariance property of \citeauthor{Monge1781} maps. For $\mathbf{s}\in\R^d$ and $r\in\R_+$ we write $L_{r,\mathbf{s}}$ for the dilation and translation map $L_{r,\mathbf{s}}(\bx) = r \bx + \mathbf{s}$. Naturally, $L^{-1}_{r,\mathbf{s}}(\bx) =  (\bx - \mathbf{s})/r=L_{1/r,-\mathbf{s}/r}(\bx)$, but also $L_{r,s}=\nabla w_{r,s}$ where $w_{r,s}(\bx):=\tfrac{r}2\|\bx\|^2-\mathbf{s}^T\bx$ is convex.
\begin{lemma}
The \citeauthor{Monge1781} map $T(\mu,\nu)$ is equivariant w.r.t to dilation and translation maps, as
$$T^\star((L_{r,\mathbf{s}})_\#\mu, (L_{r',\mathbf{s}'})_\#\nu) = L_{r',\mathbf{s}'}\circ T^\star(\mu,\nu) \circ L^{-1}_{r,\mathbf{s}}.$$
\end{lemma}
\begin{proof} Following \citeauthor{Bre91}'s theorem, let $u$ be a convex potential such that $T^\star(\mu,\nu)=\nabla u$. Set $F:=L_{r',\mathbf{s}'}\circ \nabla u \circ L^{-1}_{r,\mathbf{s}}$. Then $F$ is the composition of the gradients of 3 convex functions. Because the Jacobians of $L_{r,s}$ and $L^{-1}_{r,\mathbf{s}}$ are respectively $r\mathbf{I}_d$ and $\mathbf{I}_d/r$, they commute with the Hessian of $u$. Therefore the Jacobian of $F$ is symmetric, positive definite, and $F$ is the gradient of a convex potential that pushes $(L_{r,\mathbf{s}})_{\#}\mu$ to $(L_{r',\mathbf{s}'})_{\#}\nu$, and is therefore their \citeauthor{Monge1781} map by \citeauthor{Bre91}'s theorem.
\end{proof}

In practice this equivariance means that, when focusing on permutation matrices (which can be seen as the discrete counterparts of \citeauthor{Monge1781} maps), one is free to rescale and shift either point cloud. This remark has a practical implication when running \citeauthor{Sinkhorn64} as well. When using the squared-Euclidean distance matrix, the cost matrix is a sum of a correlation term with two rank-1 norm terms, $\mathbf{C} = -\mathbf{X}^T\mathbf{Y} + \tfrac12(\boldsymbol{\xi}\mathbf{1}_n^T + \mathbf{1}_n\boldsymbol{\gamma}^T) $ where $\boldsymbol{\xi}$ and $\boldsymbol{\gamma}$ are the vectors composed of the $n$ squared norms of vectors in $\mathbf{X}$ and $\mathbf{Y}$. Yet, due to the constraints $\mathbf{P}\mathbf{1}_n=\mathbf{a}, \mathbf{P}^T\mathbf{1}_n=\mathbf{b},$ any modification to the cost matrix of the form 
$\tilde{\mathbf{C}}=\mathbf{C}-\mathbf{c}\mathbf{1}_n^T-\mathbf{1}_n\mathbf{d}^T,$ where $\mathbf{c}, \mathbf{d}\in \mathbb{R}^n$ only shifts the \eqref{eq:entdual} objective by a constant:
$\langle \mathbf{P}, \tilde{\mathbf{C}} \rangle = \langle \mathbf{P}, \mathbf{C} \rangle - \tfrac1n\mathbf{1}_n^T\mathbf{c} - \tfrac1n\mathbf{1}_n^T\mathbf{d}$.
In practice, this means that norms only perturb~\citeauthor{Sinkhorn64} without altering the optimal coupling, and one should focus on the negative correlation matrix $\mathbf{C}:=-\mathbf{X}^T\mathbf{Y}$, replacing Line \ref{line:sink_change} in Algorithm \ref{algo:sinkhorn}. We do observe substantial stability gains of these properly rescaled costs when comparing two point clouds (see Appendix~\ref{app:bettersinkhorn}).

\textbf{Warm-starting Sinkhorn.} Solving the EOT problem \eqref{eq:entdual} from scratch for each new batch of noise-data pairs $(\bm{X}_0, \bm{X}_1)$ is generally unnecessarily costly, since the solution is discarded each time a new batch is drawn. For large batch sizes, we propose to use the OT solution to $i$th batch $(\bm{X}_0^{(i)}, \bm{X}_1^{(i)})$ by warm-starting Sinkhorn for the $(i+1)$th batch $(\bm{X}_0^{(i+1)}, \bm{X}_1^{(i+1)})$. Let $(\bm{f}^\star,\bm{g}^\star)$ be the optimal dual potentials for a given batch $(\bm{X}, \bm{Y})$. Then, these potentials can be extended to the continuous domain:
\begin{align*}
    \textstyle\bm{f}(\bm{x}) &= \varepsilon \log \tfrac{1}{n} + \min_\varepsilon(\bm{C}(\bm{x}, \bm{y}_j) - \bm{g}_j), \\ 
    \textstyle\bm{g}(\bm{y}) &= \varepsilon \log \tfrac{1}{n} + \min_\varepsilon(\bm{C}(\bm{x}_i, \bm{y}) - \bm{f}_i). 
\end{align*}
For a new batch $(\bm{X}', \bm{Y}')$, we use the above formula to initialize the potentials $(\bm{f}', \bm{g}')$, i.e. $(\bm{f}', \bm{g}') \gets (\bm{f}(\bm{x}_i')_i, \bm{g}(\bm{y}_j')_j)$. Since \eqref{eq:entdual} is strictly convex, the choice of initialization has no influence on the solution. In practice, we find that warm-starting Sinkhorn substantially reduces the number of iterations required and the overall runtime of OTFM. We ablate the role of warmstart in \Cref{app:warmstart}.

\textbf{Computing Matchings in PCA Space.} With the dot-product cost we can further use Principal Component Analysis (PCA) to optimally reduce the dimensionality of the cost matrix and significantly speed up Sinkhorn computation. Let $\bx$ and $\by$ represent noise and data samples respectively, and let $\mathbf{A} \in \R^{k \times d}$ denote the projection matrix whose rows contain top-$k$ PCA directions. The PCA reconstruction of $\by$ is $\mathbf{A}^\top\mathbf{A}\by$, and
\[
    \bx^\top \by \approx \bx^\top\mathbf{A}^\top\mathbf{A}\by = \bar{\bx}^\top\bar{\by},
\]
where $\bar{\bx}$ and $\bar{\by}$ are the projection of $\bx$ and $\by$ onto the PCA subspace. Note that we can achieve this dimensionality reduction regardless of the structure of $\bx$, and this trick can be applied in the generative setting where $\bx$ is an isotropic Gaussian vector. This reduces the naive runtime of computing the cost matrix from $n^2d$ to $n^2k$. For large $n$, we compute the cost matrix on the fly per Sinkhorn iteration and avoid materializing the entire matrix at once, hence this reduction also occurs per iteration. In our experiments, we can achieve an almost 10x speedup in Sinkhorn computation from PCA, without sacrificing generation quality; see Appendix \ref{app:pca} for details.

\textbf{Precomputing Noise/Data Pairs.} We can completely separate the computational cost of preparing coupled noise/data pairs from the cost of training the model. To do so, as $n$ datapoints are retrieved from a dataloader and $n$ Gaussian samples are drawn, we can accumulate and buffer the outputs of Steps~\ref{step:1}--\ref{step:4} of \Cref{algo:sinkhorn} in a new augmented dataloader.
To avoid storing noise vectors, we generate each noise vector using a single PsuedoRandom Number Genearator (PRNG) key, and only store pairs of data identifier and the corresponding PRNG key for the coupled noise vector.
When training an FM model (Steps~\ref{step:5}--\ref{step:8} of \Cref{alg:train}), we load pairs of data identifier and PRNG key from this new dataloader, retrieve the corresponding data, and generate the noise using the key.
We use this approach while ablating any hyperparameters of FM training, to avoid Sinkhorn recomputations.

\textbf{Scaling Up \citeauthor{Sinkhorn64} to Millions of High-Dimensional Points.}
When guiding flow matching with batch-OT as presented in Algorithm~\ref{alg:train}, our ambition is to vary $n$ and $\varepsilon$ so that the coupling $\mathbf{P}^\varepsilon$ used to sample indices can be both large ($n\approx10^6$) and sharp if needed, i.e. with an $\varepsilon$ that can be brought to arbitrarily low levels so that $\mathcal{E}(\mathbf{P}^\varepsilon)\approx 0$. To that end, we leverage the OTT-JAX implementation of the \citeauthor{Sinkhorn64} algorithm \citep{cuturi2022optimal}, which can be natively sharded across multi-GPUs, or more generally multiple nodes of GPU machines equipped with efficient interconnect. In that approach, inspired by the earlier mono-GPU implementation of~\citep{feydy2020analyse}, all $n$ points from source and target are sharded across GPUs and nodes (we have used either 1 or 2 nodes of 8 GPUs each, either NVIDIA H100 or A100). A crucial point in our implementation is that the cost matrix $\mathbf{C}=-\mathbf{X}\mathbf{Y}^T$ (following remark above) is never instantiated globally. Instead, it is recomputed at each $\min_\varepsilon$ operation in Lines~\ref{line:minf} and~\ref{line:ming} of Algorithm~\ref{algo:sinkhorn} locally, for these shards. All sharded results are then gathered to recover $\bff,\bg$ newly assigned after that iteration. When outputted, we use $\bff^\varepsilon$ and $\bg^\varepsilon$ and, analogously, never instantiate the full $\mathbf{P}^\varepsilon$ matrix (this would be impossible at sizes $n\approx 10^6$ we consider) but instead, materialize it blockwise to do stratified index sampling corresponding to Line \ref{line:sampling} in Algorithm~\ref{alg:train}. We use the Gumbel-softmax trick to vectorize the categorical sampling of each of these lines to select, for each line index $i$, the corresponding column $j_i$.

\section{Experiments}\label{sec:experiments}
We revisit the application of Algorithm~\ref{alg:train} using the modifications to the \citeauthor{Sinkhorn64} algorithm outlined in Section~\ref{sec:newsink} to consider various benchmark tasks for which I-FM has been used. We consider synthetic tasks in which the ground-truth \citeauthor{Monge1781} map is known, and benchmark unconditioned image generation using CIFAR-10~\citep{krizhevsky2009learning}, and the 32$\times$32 and 64$\times$64 downsampled variants~\citep{chrabaszcz2017downsampled} of the ImageNet dataset~\citep{deng2019imagenet}.

\textbf{\citeauthor{Sinkhorn64} Hyperparameters.}
To track accurately whether the \citeauthor{Sinkhorn64} algorithm converges for low $\varepsilon$ values, we set the maximal number of iterations to \num{50000}. We use the adaptive momentum rule introduced in~\citep{lehmann2022note} beyond $2000$ iterations to speed-up harder runs. Overall, all of the runs below converge: even for low $\varepsilon$, we achieve convergence except in a very few rare cases. The threshold $\tau$ is set to $0.001$ and we observe that it remains relevant for all dimensions, as we use the 1-norm to quantify convergence.

\subsection{Evaluation Metrics To Assess the Quality of a Flow Model $\mathbf{v}_\theta$}
All metrics used in our experiments can be interpreted as \textit{lower is better}.

\textbf{Negative log-likelihood.} Given a trained flow model $\bm{v}_\theta(t, \bm{x})$, the density $p_t(\bm{x})$ obtained by pushing forward $p_0(\bm{x})$ along the flow map of $\bm{v}_\theta$ can be computed by solving 
\begin{equation}
    \log p_t(\bm{x}_t) = \log p_0(\bm{x}_0) - \int_0^1 (\nabla_x \cdot \bm{v}_\theta)(t, \bm{x}_t) \: \mathrm{d} t, \qquad \dot{\bm{x}}_t = \bm{v}_\theta(t, \bm{x}_t), \label{eq:density_ODE}
\end{equation}
Similarly, given a pair $(t, \bm{x})$, the density $p_t(\bm{x})$ can be evaluated by backward integration \cite[Section 2.2]{grathwohl2018ffjord}. The divergence $(\nabla_x \cdot \bm{v}_\theta)(t, \bm{x}_t)$ requires computing the trace of the Jacobian of $\mathbf{v}_\theta(t,\cdot)$. As commonly done in the literature, we use the Hutchinson trace estimator with a varying number of samples to speed up that computation without materializing the entire Jacobian and use either an \texttt{Euler} solver with 50 steps for synthetic tasks or a \texttt{Dopri5} adaptive solver for image generation tasks, both implemented in the Diffrax toolbox~\citep{kidger2021on}. Given $n$ points $\bm{x}_1^1,\dots,\bm{x}_1^n\sim\nu$, the negative log-likelihood (NLL) of that set is 
$$\mathcal{L}(\theta) := -\frac1n \sum_{i = 1}^n \log p_1(\bm{x}^i_1).$$ subject to \eqref{eq:density_ODE}. We alternatively report the bits per dimension (BPD) statistic, given by $\mathrm{BPD} = \mathcal{L} / (d\log 2)$.

\textbf{Curvature.} We use the \emph{curvature} of the field $\bm{v}_\theta$ as defined by \citep{lee2023minimizing}: for $n$ integrated trajectories $(\bm{x}_t^{1}, \dots, \bm{x}_t^n)$ starting from samples at $t=0$ from $\mu$, the curvature is defined as 
$$\kappa(\theta) := \frac1n \sum_{i = 1}^N \int_0^1 \| \bm{v}_\theta(t, \bm{x}_t^{i}) - (\bm{x}_1^{i} - \bm{x}_0^{i}) \|_2^2 \mathrm{d} t,$$
where the integration is done with an \texttt{Euler} solver with 50 steps for synthetic tasks and the \texttt{Dopri5} solver evaluated on a grid of 8 steps for image generation tasks. The smaller the curvature, the more the ODE path looks like a straight line, and should be easy to integrate.

\textbf{Reconstruction loss.} For synthetic tasks in Sections~\ref{subsec:exp_synth1}, we have access to the ground-truth transport map $T_0$ that generated the target measure $\mu_1$ from $\mu_0$. In both cases, that map is parameterized as the gradient of a convex \citeauthor{Bre91} potential, respectively a piecewise quadratic function and an input convex neural network, ICNN~\citep{amos2017input}. For a starting point $\mathbf{x}_0$, we can therefore compute a \textit{reconstruction loss} (a variant of the $\mathcal{L}^2$-UVP in~\cite{korotin2021neural}) as the squared norm of the difference between the true map $T^\star(\mathbf{x}_0)$ and the flow map $T_\theta$ obtained by integrating $\bm{v}_\theta(t, \cdot)$ (using a varying number of steps with a \texttt{Euler} solver or with the \texttt{Dopri5} solver), defined using $n$ points sampled from $\mu$ as
$$\mathcal{R}(\theta) := \frac1n \sum_{i = 1}^n \: \| T_\theta(\bm{x}^i_0) - T_0(\bm{x}^i_0) \|_2^2\,.$$ 

\textbf{FID.} We report the FID metric \citep{heusel2017gans} in image generation tasks. For CIFAR-10 we use the train dataset of \num{50000} images, for ImageNet-32 and ImageNet-64  we subset a random set of \num{50000} images from the train set. For generation we consider four integration solvers, \texttt{Euler} with 4, 8 and 16 steps and a \texttt{Dopri5} solver from the Diffrax library~\citep{kidger2021on}.

\subsection{Synthetic Benchmark Tasks, $d=32\sim 256$}\label{subsec:exp_synth1}
We consider in this section synthetic benchmarks of medium dimensionality ($d=32\sim256$). We favor this synthetic setting over other data sources with similar dimensions (e.g. PCA reduced single-cell data~\citep{bunne2024optimal}) in order to have access to the ground-truth reconstruction loss, which helps elucidate the impact of OT batch size $n$ and $\varepsilon$.

\textbf{Piecewise Affine \citeauthor{Bre91} Map.} The source is a standard Gaussian and the target is obtained by mapping it through the gradient of a potential, itself a (convex) piecewise quadratic function obtained using the pointwise maximum of $k$ rank-deficient parabolas:
\begin{equation}\label{eq:quadpots}
u(\bx):=\max_{i\leq k} u_i(\bx) := \tfrac12 \|\bx\|^2 + \tfrac12\|\mathbf{A}_i(\bx-\mathbf{m}_i)\|^2 - \|\mathbf{A}_i\mathbf{m}_i\|^2\,,
\end{equation}
 where $\mathbf{A}_i\sim \textrm{Wishart}(\tfrac{d}{2},I_d), \mathbf{m}_i\sim \mathcal{N}(0,3I_d), c_i\sim \mathcal{N}(0,1)$ and all means are centered around zero after sampling. In practice, this yields a transport map of the form $\nabla u(\bx)= \bx + \mathbf{A}_{i^\star}(\bx-\mathbf{m}_{i^\star})$ where $i^\star$ is the potential selected for that particular $\bx$ (i.e. the argmax in \eqref{eq:quadpots}). The correction $- \|\mathbf{A}_i\mathbf{m}_i\|^2$ is designed to ensure that these potentials are sampled equally even when $\mathbf{m}_i$ is sampled far from 0. The number of potentials $k$ is set to $d/16$. Examples of this map are shown in Appendix~\ref{app:example_gaussian} for dimension $128$. We consider this setting in dimensions $d=32,64,128,256$.

\textbf{\citeauthor{korotin2021neural} Benchmark.}
The source is a predefined Gaussian mixture and the ground-truth OT map is a pre-trained ICNN. We consider this benchmark in various dimensions $d=32\sim256$, using their ICNN checkpoints. This problem is more challenging than the previous one, because both the source \emph{and} target distributions have multiple modes, and the OT map itself is a fairly complex ICNN.

\begin{wrapfigure}{r}{0.45\textwidth}
    \centering
    \begin{minipage}[c]{\linewidth}
    \centering
    \hspace{1cm}
    \input{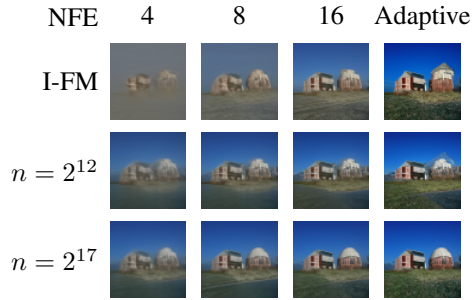}    
    \caption{Samples generated from models trained on \textbf{ImageNet-64}. $n$ denotes the total OT batch size. We use $\varepsilon=0.1$ and the \texttt{Euler} solver (\texttt{Dopri5} for adaptive with NFE $\approx 270$). More samples provided in \Cref{fig:imagenet64_grid}.}
    \label{fig:imagenet64_grid_samples}
    \end{minipage}
    \vspace{-5mm}
\end{wrapfigure}
\textbf{Velocity Field Parameterization and Training.}
The velocity fields are parameterized as MLPs with 5 hidden layers, each of size $512$ when $d=32,64$ and $1024$ when $d=128,256$. Time in $[0,1]$ is encoded using $d/8$ Fourier encodings. All models are trained with unpaired batches: the sampling in Line \ref{line:x0} of Algorithm~\ref{alg:train} is done as $\mathbf{X}_0\sim \mu$ while for Line \ref{line:x1}, $\mathbf{X}_1:=T_0(\mathbf{X}_0')$ where $\mathbf{X}_0'$ is a new sample from $\mu$ and $T$ is applied to each of the $n$ points described in $\mathbf{X}_0'$. All models are trained for 8192 steps, with effective batch sizes of 2048 samples (256 per GPU) to average a gradient, a learning rate of $10^{-3}$ (we tested with $10^{-2}$ or $10^{-4}$, the former was unstable while the latter was less efficient on a subset of runs). The model marked as \textcolor{ForestGreen}{$\blacktriangle$} in the plots is a flow model trained with \textit{perfect} supervision, i.e. given \textit{ground-truth paired samples} $\mathbf{X}_0\sim \mu$ and $\mathbf{X}_1:=T_0(\mathbf{X}_0)$, provided in the correct order. I-FM is marked as \textcolor{red}{$\blacktriangledown$}. For all other runs, we vary $\varepsilon$ (reporting renormalized entropy $\mathcal{E}(\mathbf{P}^\varepsilon)$) and the total batch size $n$ used to compute couplings, somewhere between $2048$ and \num{2097152}. These runs are carried out on a single node with 8 GPUs, and therefore the data is sharded in blocks of size $n/8$ when running the \citeauthor{Sinkhorn64} algorithm.

\textbf{Results.} The results displayed in Figures \ref{fig:piecewise} and \ref{fig:korotin} paint a homogeneous picture: as can be expected, increasing $n$ is generally impactful and beneficial for all metrics. The interest of decreasing $\varepsilon$, while beneficial in smaller dimensions, can be less pronounced in higher dimensions. Indeed, we find that renormalized entropies around $\approx 0.1$ should be advocated, if one has in mind the computational effort needed to get these samples, pictured at the bottom of each figure.

\begin{figure}
\centering
\!\!\!\!\!\!\!\includegraphics[width=1.05\linewidth]{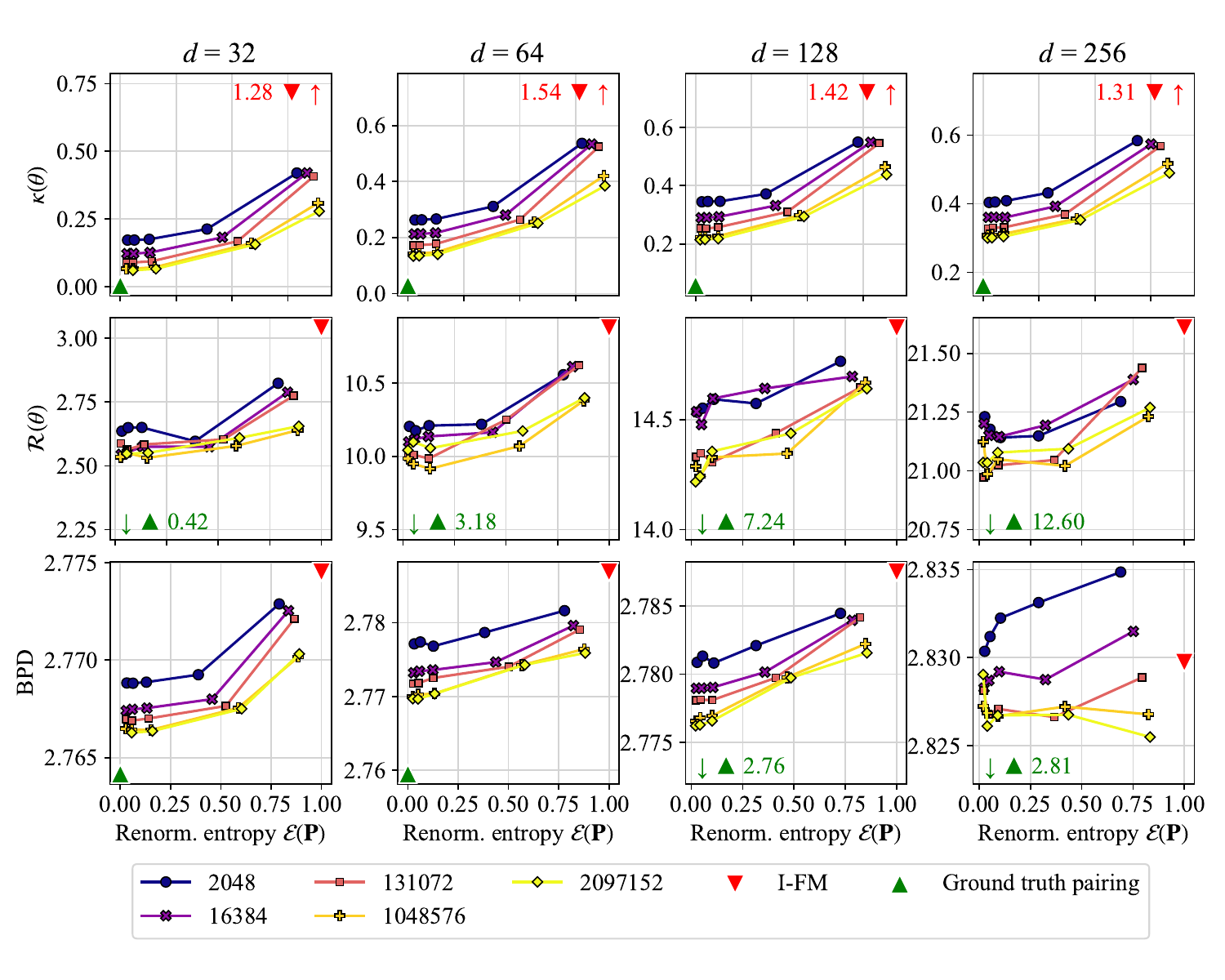} 
\vskip-.4cm
\includegraphics[width=\linewidth]{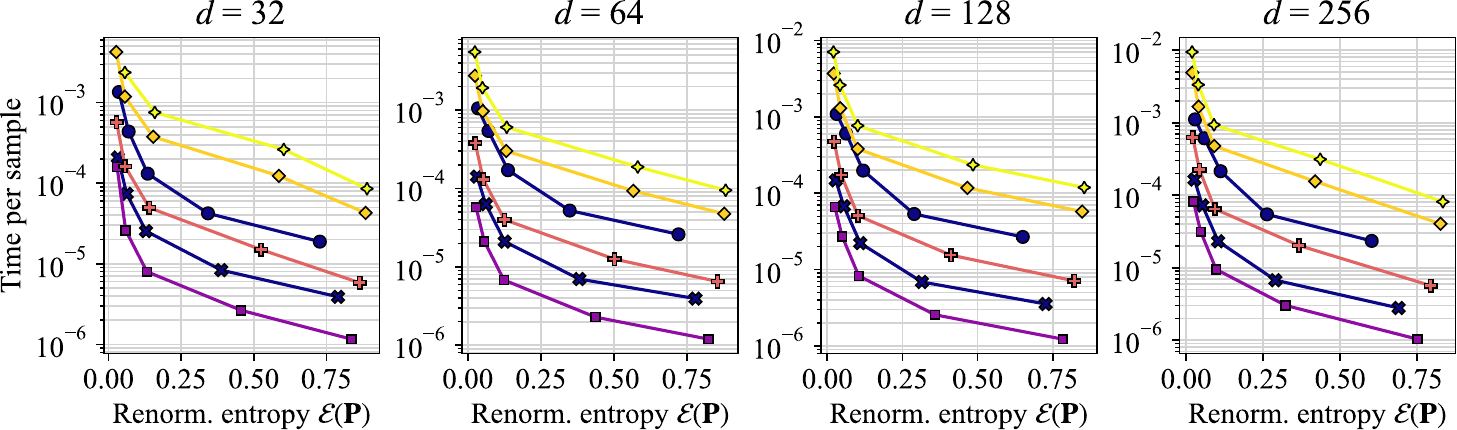} 
\caption{Results on the \textbf{piecewise affine OT Map benchmark}. The three top rows present (in that order) curvature, reconstruction and BPD metrics. Below, we provide compute times associated with running the \citeauthor{Sinkhorn64} algorithm as a per-example cost. This per-example cost is the total time needed to run \citeauthor{Sinkhorn64} to get $n\times n$ coupling divided by $n$. That cost would be 0 when using I-FM. We observe across all dimensions improvements of all metrics.}\label{fig:piecewise}
\end{figure}

\begin{figure}
\!\!\!\!\!\!\!\includegraphics[width=1.05\linewidth]{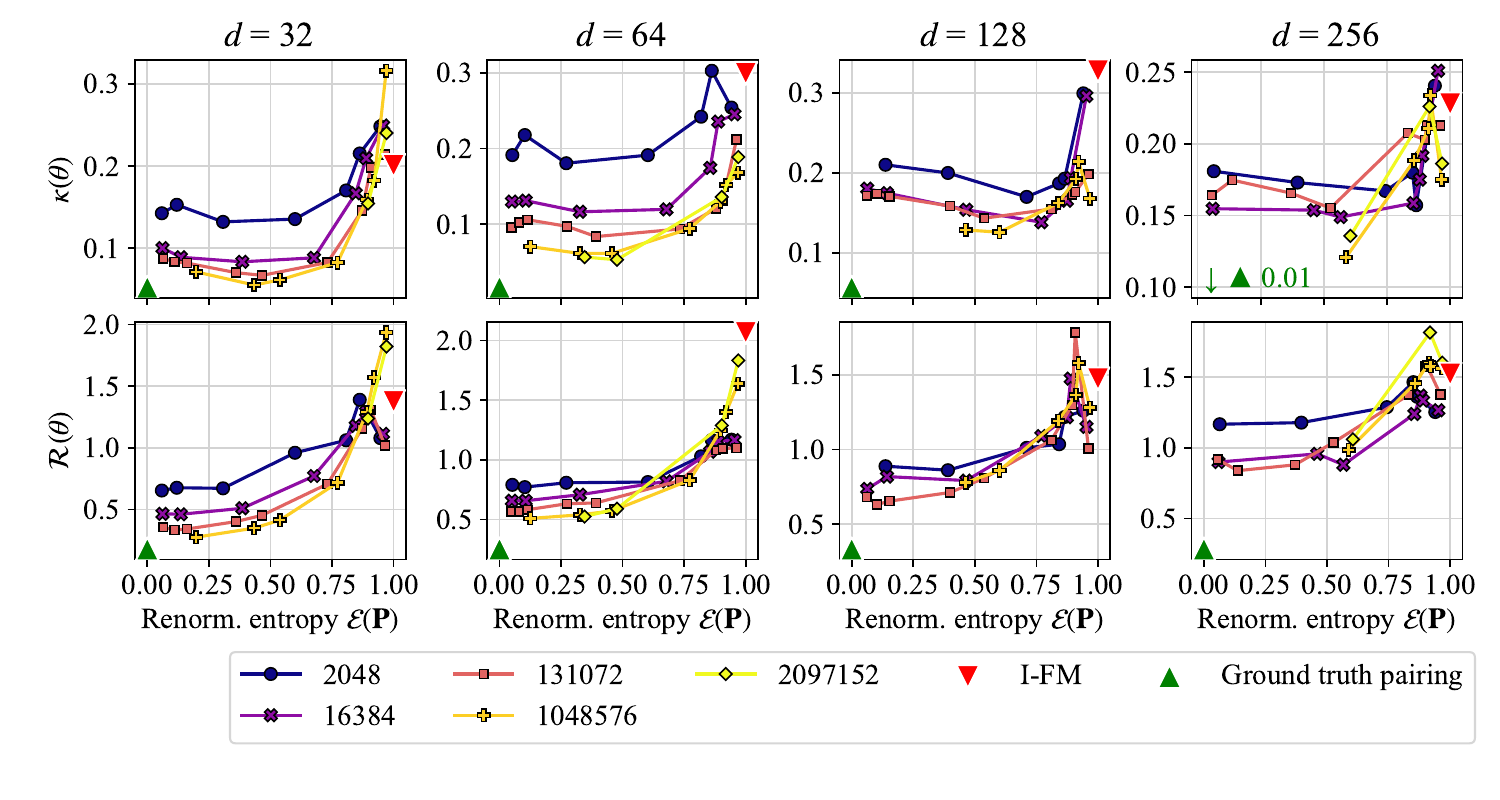}
\vskip-.4cm
\includegraphics[width=\linewidth]{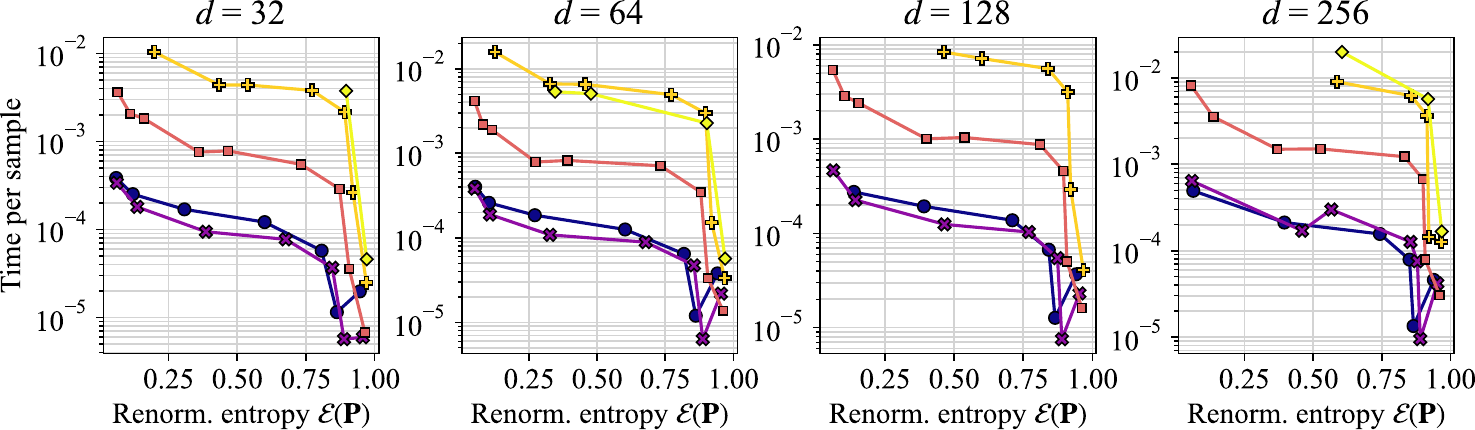} 
\caption{Results on the \textbf{Korotin benchmark}. As with Figure~\ref{fig:piecewise}, we compute curvature and     reconstruction metrics, and compute times below. Some of the runs for largest OT batch sizes $n$ are provided in the supplementary. These runs suggest that to train OT models in these dimensions increasing $n$ is overall beneficial across the board.}\label{fig:korotin}
\end{figure}

\subsection{Unconditioned Image Generation, $d=3072 \sim 12288$.}
\begin{table}[t]
    \renewcommand{\arraystretch}{1.3}
    \centering
    \begin{tabular}{cccccccccc}
        \toprule
        \multicolumn{5}{c}{\textbf{ImageNet-32}} & \multicolumn{5}{c}{\textbf{ImageNet-64}}\\
        \cmidrule(lr){1-5} \cmidrule(lr){6-10}
        \makecell{NFE $\rightarrow$ \\ $n \downarrow$} & 4 & 8 & 16 & \makecell{Adaptive \\ $115 \pm 1$} & \makecell{NFE $\rightarrow$ \\ $n \downarrow$} & 4 & 8 & 16 & \makecell{Adaptive \\ $269 \pm 1$}
        \\
        \hline
        I-FM & $66.4$ & $24.3$ & $12.1$ & $5.55$ & I-FM & $80.1$ & $37.0$ & $19.5$ & $9.32$ \\
        $2048$ & $38.2$ & $16.8$ & $10.0$ & $5.89$ & $4096$ & $50.3$ & $25.0$ & $15.8$ & $9.39$\\
        $65536$ & $33.1$ & $15.1$ & $9.28$ & $4.88$ & $32768$ & $48.8$ & $24.6$ & $15.7$ & $9.08$\\
        $524288$ & $\mathbf{31.5}$ & $\mathbf{14.8}$ & $\mathbf{9.19}$ & $\mathbf{4.85}$ & $131072$ & $\mathbf{46.9}$ & $\mathbf{23.9}$ & $\mathbf{15.4}$ & $\mathbf{8.99}$\\
        \bottomrule
    \end{tabular}
    \vspace{2mm}
    \caption{FID for models trained across different OT batch sizes. We use the best checkpoint (w.r.t FID at \texttt{Dopri5}) for each model, restricting results to the setting where the relative \texttt{epsilon} value $\varepsilon=0.1$ for ease of presentation (more detailed results can be seen in the plots of Figure~\ref{fig:imagenet32_results}).}
    \label{tab:fid_main}
    \vspace*{-5mm}
\end{table}

As done originally in~\citep{lipman2023flow}, we consider unconditional generation of the CIFAR-10, ImageNet-32 and ImageNet-64 datasets.

\textbf{Velocity Field Parameterization and Training.} We use the network parameterization given in~\citep{tong2024improving} for CIFAR-10 and those given in~\citep{pooladian2023multisample} for ImageNet 32 and ImageNet 64. We follow their recommendations on setting learning rates, batch sizes (to average gradients) as well as total number of iterations: we train respectively for 400k, 438k and 957k  using effective batch sizes advocated in their paper, respectively $16 \times 8$, $128\times 8$ and $50\times 16$. 

\textbf{CIFAR-10.} Results are presented in Figure~\ref{fig:cifar10}. Compared to results reported in~\citep{tong2023simulation} we observe slightly better FID scores (about 0.1) for both I-FM and Batch OT-FM. Note that the size of the dataset itself (50k, 100k when including random flipping as we do) is comparable (if not slightly lower) to our largest batch size $n=\num{131072}$, meaning some images are duplicated. Overall, the results show the benefit of relatively larger batch sizes and suitably small $\varepsilon$, that is more pronounced at lower NFE.

\textbf{ImageNet-32 and ImageNet-64.} Results are shown in Figures~\ref{fig:imagenet32_results} and ~\ref{fig:imagenet64_results}. Compared to results reported in~\citep{tong2023simulation} we observe slightly better FID scores (about 0.1 when using the \texttt{Dopri5} solver for instance) for I-FM. Compared to CIFAR-10, these datasets are more suitable for our large OT batch sizes as they contain significantly more samples, and we continue to observe the benefits of larger batch size and proper choice of renormalized entropy.

\begin{figure}
\includegraphics[width=\linewidth]{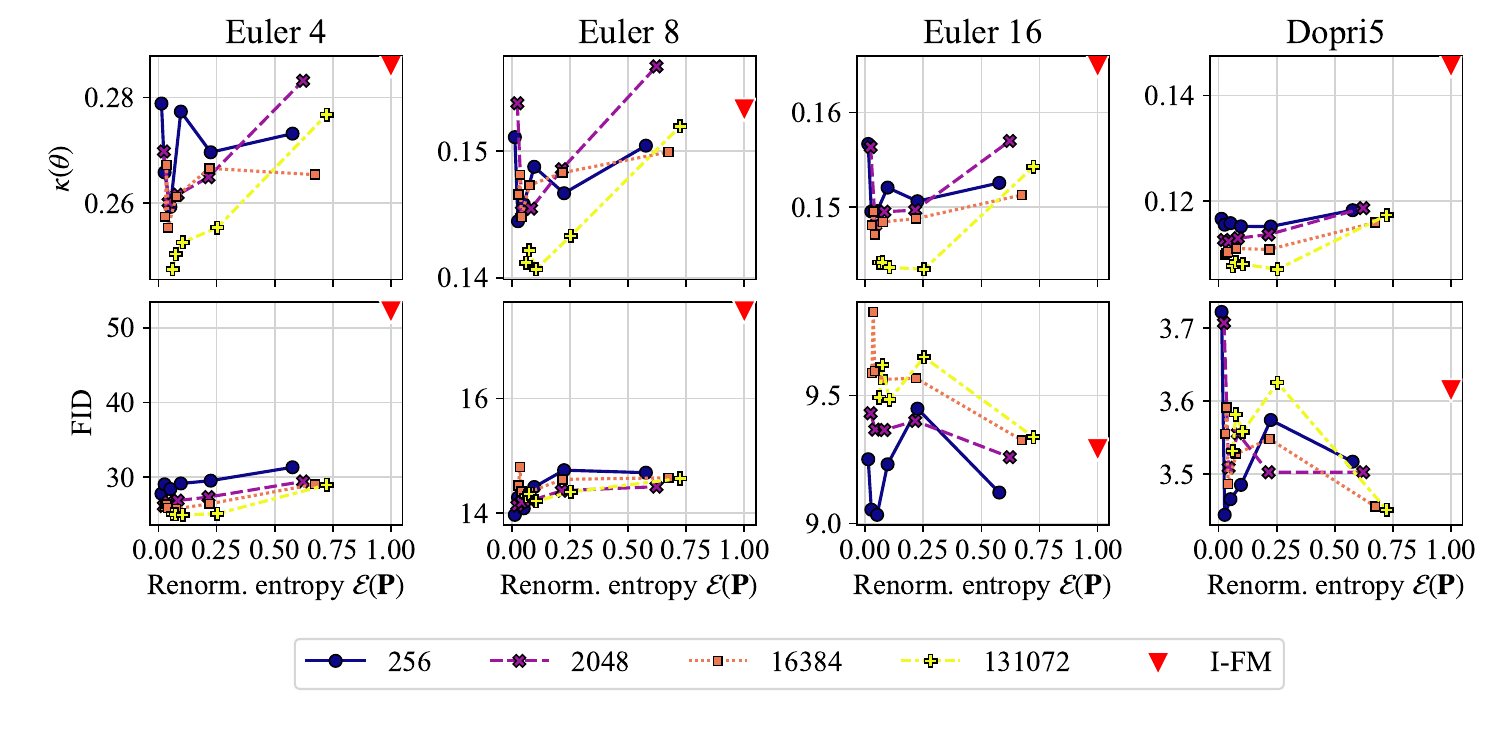} 
\vskip-.7cm
\caption{Experiment metrics for \textbf{CIFAR-10} image generation. We evaluate the trained models using the \texttt{Euler} solver with three different number of steps, and with the \texttt{Dopri5} solver and adaptive steps. The plots demonstrate the benefits of a larger OT batch size to achieve significantly smaller curvature, and moderately smaller FID at low number of integration steps. CIFAR-10 is not necessarily the best setup to evaluate the performance of OT based FM, since the number of points is relatively low (the batch sizes we consider involve in fact resampling \textit{data}). Our experiments also suggest that in this setting, lower renormalized entropy generally benefits the performance.}\label{fig:cifar10}
\end{figure}

\begin{figure}[h]
\vskip-.3cm
\includegraphics[width=\linewidth]{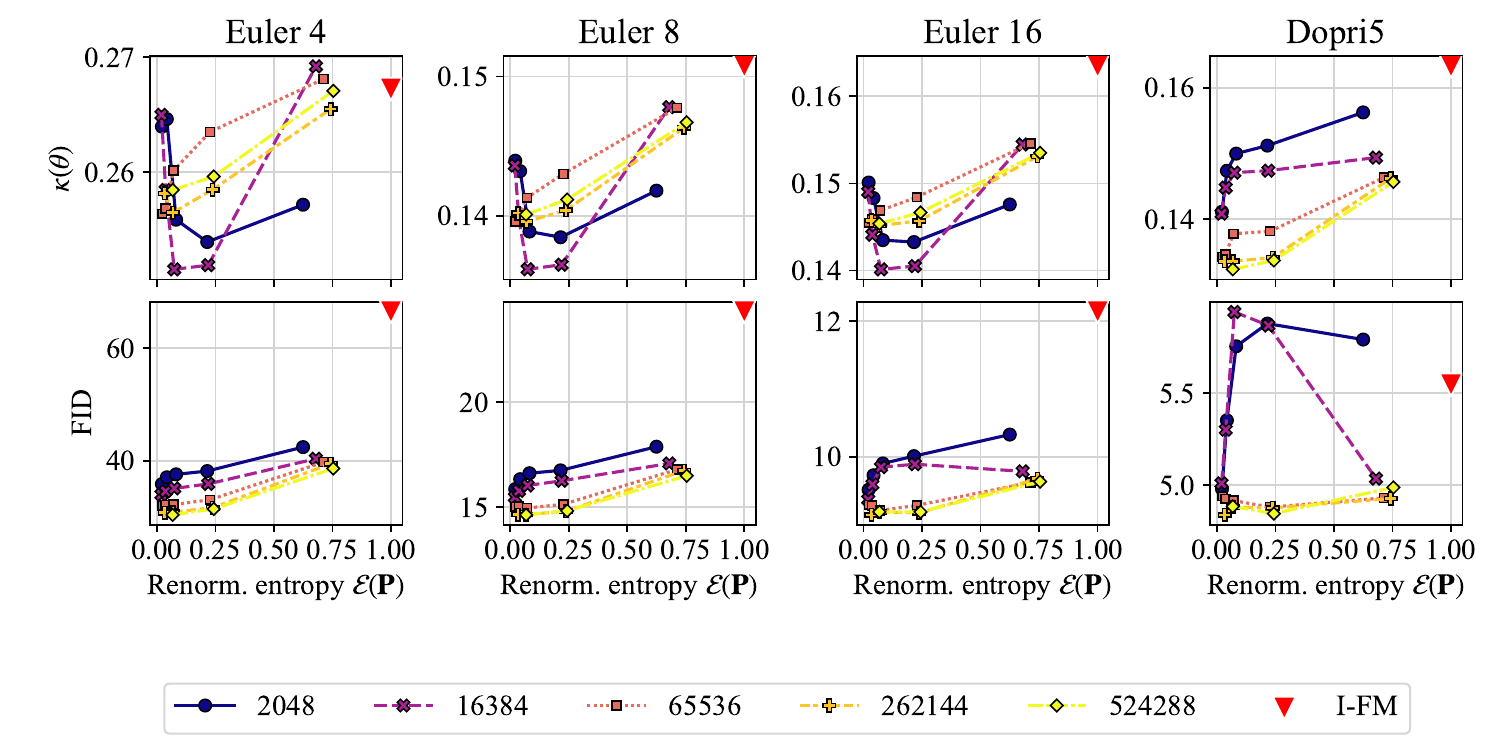} 
\vskip-.2cm
\caption{\textbf{ImageNet-32} experiment metrics. We observe that both FID and curvature are smaller when using larger OT batch size, and smaller renormalized entropy tends to result in better metrics.}\label{fig:imagenet32_results}
\end{figure}

\begin{figure}[h]
\includegraphics[width=\linewidth]{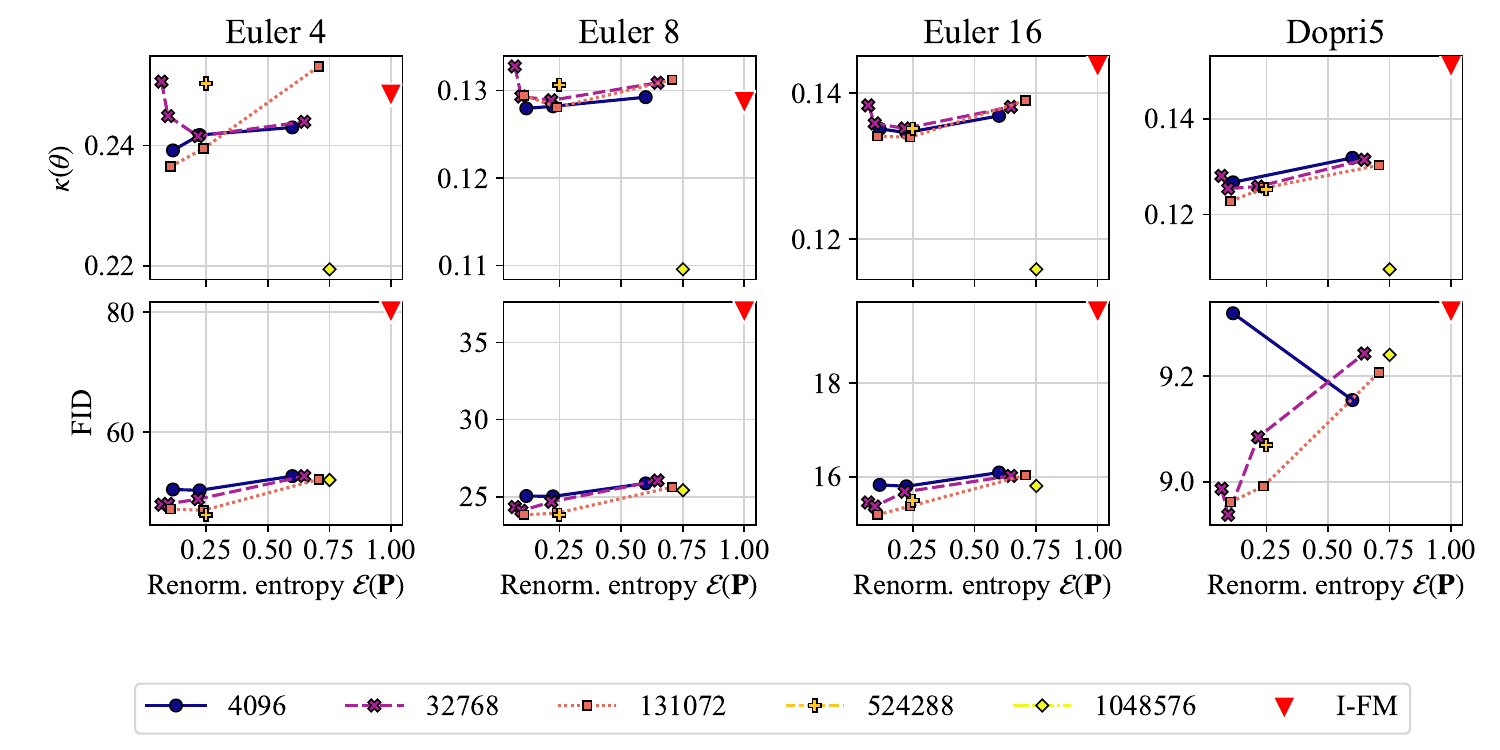} 
\caption{\textbf{ImageNet-64} results: Curvature and FID obtained with \texttt{Euler} integration with varying number of steps, as well as \texttt{Dopri5} integration.}
\label{fig:imagenet64_results}
\end{figure}

\subsection*{Conclusion} Our experiments suggest that guiding flow models with large scale \citeauthor{Sinkhorn64} couplings can prove beneficial for downstream performance. We have tested this hypothesis by computing and sampling from both crisp and blurry $n\times n$ \citeauthor{Sinkhorn64} coupling matrices for sizes $n$ in the millions of points, placing them on an intuitive scale from 0 (close to using an optimal permutation as returned e.g. by the Hungarian algorithm) to 1 (equivalent to the independent sampling approach popularized by~\citet{lipman2023flow}). This involved efficient multi-GPU parallelization, realizing scales which, to our knowledge, were never achieved previously in the literature. Although the scale of these computations may seem large, they are still relatively cheap compared to the price one has to pay to optimize the FM loss, and, additionally, are completely independent from model training. As a result, they should be carried out prior to any training. While we have not explored the possibility of launching multiple jobs with them (to ablate, e.g., for other fundamental aspects of model training such as learning rates), we leave a more careful tuning of these training runs for future work. We claim that paying this relatively small price to log and sample paired indices obtained from large scale couplings results for mid-sized problems in great returns in the form of faster training and faster inference, thanks to the straightness of the flows learned with the batch-OT procedure. For larger sized problems, the conclusion is not so clear, although we quickly observe benefits when using middle values for $n$ (in the thousands) that typically improve when going beyond hundreds of thousands when relevant, and renormalized entropies of around $0.1$. This forms our main practical recommendation for end users.

\textbf{Limitations.} Our results rely on training of neural networks. In the interest of comparison, we have used the same model across all changes advocated in the paper (on $n$ and $\varepsilon$). However, and due to the scale of our experiments, we have not been able to ablate important parameters such as learning rates when varying $n$ and $\varepsilon$, and instead relied on those previously proposed for I-FM.

\bibliography{biblio}
\bibliographystyle{plainnat}

\newpage
\appendix

\section{Appendix}

\subsection{The Necessity of Large OT Batch Size}\label{app:proof}
Here, we formalize the assumptions in and provide the proof of \Cref{prop:var_lower_bound_informal}.

Assuming that $\mu_0$ admits a density, if we were to couple $X_0$ and $X_1$ through optimal transport, by \Cref{thm:brenier_thm} we would have $\Var(X_1 \,|\, X_0) = 0$ a.s.\ over $X_0$, where variance is the sum of coordinate variances. In general, any coupling that provides $\Var(X_1 \,|\, X_0) = 0$ allows for one-step generation, simply by performing least-squares regression to learn $\E[X_1 \,|\, X_0]$.
Therefore, we adopt $\Var(X_1 \,|\, X_0)$ as a measure of success of a coupling.

Recall from \ref{alg:train} that to obtain a pair of samples $X_0,X_1$ for training, we first draw n i.i.d.\ samples $\bX_0 \sim \mu_0^{\otimes n}$ and $\bX_1 \sim \mu_1^{\otimes n}$. Then, we sample $X_0,X_1 \sim \hat{\pi}_n(\bX_0,\bX_1)$, where $\hat{\pi}_n$ denotes the discrete optimal (entropic) transport solution between the uniform distribution on $\bX_0$ and $\bX_1$. We only require $\hat{\pi}_n(\bX_0,\bX_1)$ to be supported on $\bX_0 \times \bX_1$, as formalized by the following assumption.

\begin{assumption}\label{assump:support}
    $\hat{\pi}_n(\bX_0,\bX_1)$ is supported on $\bX_0$ and $\bX_1$, more precisely,
    \[
    \hat{\pi}_n(\bX_0,\bX_1) = \sum_{i,j}^{n}P_{ij}(\bX_0,\bX_1)\delta_{(\bx_0^{(i)},\bx_1^{(j)})},
    \]
    where $(P_{ij}(\bX_0,\bX_1))_{ij}$ is some bistochastic matrix, equivariant under permutations of $\bX_0$ and $\bX_1$, and $\delta$ denotes the Dirac measure.
\end{assumption}

To capture the intrinsic dimension of data, we can impose the following assumption on $\mu_1$.
\begin{assumption}\label{assump:data}
    For $X$ and $X'$ drawn independently from the data distribution $\mu_1$, we have
    $$\Prob[\Vert X - X' \Vert \leq t] \leq Ct^r,$$
    for all $t > 0$ and some $C,r > 0$.
\end{assumption}
Note that the volume of an $r$-dimensional ball of radius $t$ is proportional to $t^r$. Therefore, $r$ in the above assumption roughly captures the intrinsic dimension of data, typically assumed to be much less than the ambient dimension, i.e. $r \ll d$.

We are now ready to present the proof of \Cref{prop:var_lower_bound_informal}, which we repeat here for ease of reference.

\begin{proposition}\label{prop:var_lower_bound}
    Suppose $\hat{\pi}_n$ is any coupling rule that satisfies Assumption \ref{assump:support}, and that $\mu_1$ satisfies Assumption \ref{assump:data}. Define the coupling $X_0,X_1 \sim \pi_n$ as follows: first draw $\bX_0 \sim \mu_0^{\otimes n}$ and $\bX_1 \sim \mu_1^{\otimes n}$, then sample $X_0,X_1 \sim \hat{\pi}_n(\bX_0,\bX_1)$. Then, for any $\bx_0 \in \sR^d$, we have
    \[
    \Var_{X_0,X_1 \sim \pi_n}(X_1 \,|\, X_0 = \bx_0) \geq cn^{-2/r},
    \]
    where $c > 0$ is a constant depending only on $C$ and $r$.
\end{proposition}

To prove \Cref{prop:var_lower_bound}, we use the fact that
\[
    \Var(X_1 \,|\, X_0 = \bx_0) = \frac{1}{2}\E[\Vert X_1 - X'_1\Vert^2 \,|\, X_0 = \bx_0],
\]
where $X_1$ and $X'_1$ are drawn independently from $\pi_n(\cdot \,|\, X_0 = \bx_0)$. However, $X_1$ and $X'_1$ essentially come from different batches $\bX_1$ and $\bX'_1$, and their only dependence is through being coupled with $X_0 = \bx_0$. We can remove this dependence by lower bounding the variance by the minimum distance between two batches of i.i.d.\ samples $\bX_1$ and $\bX'_1$. This is performed by the following lemma.

\begin{lemma}
    Let $\pi_n$ be as defined in \Cref{prop:var_lower_bound}. Then, for any $\bx_0 \in \sR^d$, we have
    \[
    \Var_{X_0,X_1 \sim \pi_n}(X_1 \,|\, X_0 = \bx_0) \geq \frac{1}{2}\E_{\bX_1,\bX'_1 \sim \mu_1^{\otimes n} \otimes \mu_1^{\otimes n}}[D(\bX_1,\bX'_1)],
    \]
    where $D(\bX_1,\bX'_1) \coloneqq \min_{\bx_1 \in \bX_1,\bx'_1 \in \bX'_1}\Vert \bx_1 - \bx'_1\Vert^2$.
\end{lemma}
\begin{proof}
    To draw $X_1,X'_1 \sim \pi_n(\cdot \,|\, X_0 = \bx_0) \otimes \pi_n(\cdot \,|\, X_0 = \bx_0)$ we can draw 
    $\bX_0,\bX'_0 \sim \mu_0^{\otimes n}\otimes \mu_0^{\otimes n}$ and $\bX_1,\bX'_1 \sim \mu_1^{\otimes n} \otimes \mu_1^{\otimes n}$. We then replace the first sample in $\bX_0$ and $\bX'_0$ with $\bx_0$ to condition on $X_0 = \bx_0$ (in particular, we rely on the equivariance of $\hat{\pi}_n$), and denote them by $\bX_0(\bx_0)$ and $\bX'_0(\bx_0)$. We can then write
    \begin{align*}
        \Var(X_1 \,|\, X_0 = \bx_0) &= \frac{1}{2}\E[\Vert X_1 - X'_1 \Vert^2 \,|\, X_0 = \bx_0]\\
        &= \frac{1}{2}\E\big[\E[\Vert X_1 - X'_1 \Vert^2 \,|\, \bX_0(X_0), \bX'_0(X_0), \bX_1, \bX'_1] \,|\, X_0 = \bx_0\big] \\
        &= \frac{1}{2}\E\left[\sum_{\bx_1 \in \bX_1,\bx'_1 \in \bX'_1}\hat{\pi}_n(\bX_0(\bx_0),\bX_1)(\bx_1 \,|\, \bx_0)\hat{\pi}_n(\bX'_0(\bx_0),\bX'_1)(\bx'_1  \,|\, \bx_0)\Vert \bx_1 - \bx'_1\Vert^2\right]\\
        &\geq \frac{1}{2}\E\left[D(\bX_1,\bX'_1)\sum_{\bx_1\in\bX_1,\bx'_1 \in \bX'_1}\hat{\pi}_n(\bX_0(\bx_0),\bX_1)(\bx_1\,|\,\bx_0)\hat{\pi}_n(\bX'_0(\bx_0),\bX'_1)(\bx'_1\,|\,\bx_0)\right]\\
        &\geq \frac{1}{2}\E[D(\bX_1,\bX'_1)],
    \end{align*}
    which finishes the proof.
\end{proof}

Using the above lemma, to prove \Cref{prop:var_lower_bound}, we only need to estimate the expected distance between two batches of samples from $\mu_1$.
\begin{proof}[Proof of \Cref{prop:var_lower_bound}]
    Let $\bX_1,\bX'_1$ be independent batches of $n$ i.i.d.\ samples from $\mu_1$. We use expand our notation by letting $D(\bX_1,\bx'_1) \coloneqq \min_{\bx_1 \in \bX_1}\Vert \bx_1 - \bx'_1\Vert$ be the distance between a single sample and a batch. By the Markov inequality, for any $t > 0$ we have
    \begin{align*}
        \E[D(\bX_1,\bX'_1)] &\geq t \Prob[D(\bX_1,\bX'_1) \geq t]\\
        &= t\E\Big[\Prob\Big[\bigcap_{X'_1 \in \bX'_1}\{D(\bX_1,X'_1) \geq t\} \,|\, \bX_1\Big]\Big]\\
        &= t \E\Big[\Prob[D(\bX_1, X'_1) \geq t \,|\, \bX_1]^n\Big] \tag{Independence}\\
        &\geq t\Prob[D(\bX_1, X'_1) \geq t]^n \tag{Jensen's Inequality}\\
        &= t\E\Big[\Prob\Big[D(\bX_1,X'_1) \geq t \,|\, X'_1\Big]\Big]^n\\
        &= t\E\Big[\Prob\Big[\bigcap_{X_1 \in \bX_1}\{\Vert X_1 - X'_1 \Vert \geq t\} \,|\, X'_1\Big]\Big]^n\\
        &= t\E\Big[\Prob[\Vert X_1 - X'_1\Vert \geq t \,|\, X'_1]^n]^n\\
        &\geq t\Prob[\Vert X_1 - X'_1 \Vert \geq t]^{n^2} \tag{Jensen's Inequality}\\
        &\geq t(1 - Ct^r)^{n^2}\tag{Assumption \ref{assump:data}}.
    \end{align*}
    Choosing $t = (2Cn^2)^{-1/r}$ and using the inequality $(1-1/(2x))^x \geq 1/2$ for all $x \geq 1$ yields $\E[D(\bX_1,\bX'_1)] \geq (2Cn^2)^{-1/r}/2$, which completes the proof.
\end{proof}

\subsection{Using the negative dot-product cost rather than squared-Euclidean in Sinkhorn}\label{app:bettersinkhorn}
As we mention in the main text, entropically regularized optimal transport plan for the squared Euclidean cost can be equivalently recast using exclusively the negative scalar product $(\bm{x}, \bm{y}) \mapsto -\langle \bm{x}, \bm{y} \rangle$ between source and target, and not on any absolute measure of scale. To see this, consider an affine map $\overline{\bm{x}} = \alpha \bm{x} + \bm{\beta}$ with $\alpha > 0$. Then:
\[
    \langle \overline{\bm{x}}, \bm{y} \rangle = \alpha \langle \bm{x}, \bm{y} \rangle + \langle \bm{\beta}, \bm{y} \rangle.
\]
The second term is a rank-1 term that will be absorbed by the optimal dual potentials (see \eqref{eq:entdual}) and the factor $\alpha$ amounts to a rescaling of the entropic regularization level $\varepsilon$. In particular, when there is only translation, then $\alpha = 1$ and the transport plans are identical for the same $\varepsilon$.

Therefore for input data $\bm{X} \in \mathbb{R}^{n \times d}, \bm{Y} \in \mathbb{R}^{m \times d}$ the Sinkhorn transport plan depends only on the dot-product cost $-\bm{X} \bm{Y}^T$. We argue that it is always more natural to use the dot-product cost than the full squared-Euclidean cost, and we find that in practice using directly the dot-product can improve the numerical conditioning of the Sinkhorn algorithm. This is because we drop terms arising from the squared norm, which can be very large. 
This becomes especially important for single-precision floating point computations, as is the case for the large scale GPU applications we consider. 

We illustrate this in Figure \ref{fig:sinkhorn_iters}: we sample $N =  8192$ points $\{ \bm{x}_i \}_{i = 1}^N$ in dimension $d = 128$ from the Gaussian example described in Section \ref{subsec:exp_synth1} and map them through the piecewise affine Brenier map, i.e. $\bm{y}_i = T(\bm{x}_i)$. We then introduce a translation, $\overline{\bm{y}}_i = \bm{y} + 5$. We use the Sinkhorn algorithm (Algorithm \ref{algo:sinkhorn}) with either the dot-product or squared Euclidean cost to compute the transport plan and record the number of iterations taken, and our computations are carried out on GPU with single-precision arithmetic. Even though we already use log-domain computation tricks to prevent under/overflow, we find that for small $\varepsilon$, Sinkhorn with squared Euclidean cost begins to suffer from numerical issues and fails to converge within the iteration limit of 50,000. 

\begin{figure}[h]
    \centering
    \includegraphics[width=0.5\linewidth]{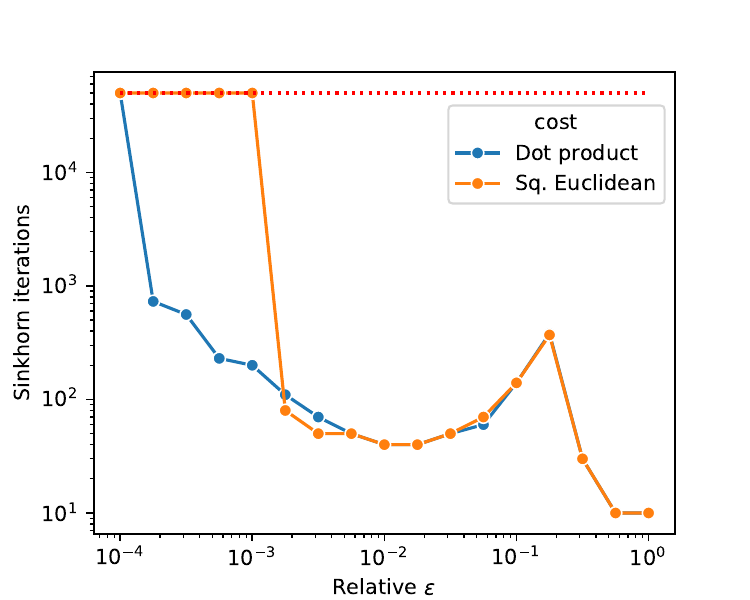}
    \caption{Number of Sinkhorn iterations against relative $\varepsilon$ for Gaussian piecewise affine OT example.}
    \label{fig:sinkhorn_iters}
\end{figure}

\subsection{Sinkhorn Speedup from Warmstart}\label{app:warmstart}

\Cref{tab:warmstart} shows the average time in seconds spent solving \eqref{eq:entdual} using Sinkhorn iterations, for the values of regularization level $\varepsilon$ and batch size $n$ we consider. We find that for almost all choices of $(n, \varepsilon)$, warmstarting yields significant speedups compared to the default Sinkhorn initialization. We therefore enable warmstart by default in our experiments. 

\begin{table}
\small
\centering 
\begin{tabular}{llllllllll}
\toprule 
 & \multicolumn{4}{c}{\textbf{Sinkhorn time, warmstart (s)}} & \multicolumn{4}{c}{\textbf{Sinkhorn time, no warmstart (s)}} \\ 
 \cmidrule{2-5} \cmidrule{6-10}
Batch size & 16384 & 65536 & 262144 & 524288 & 16384 & 65536 & 262144 & 524288 \\
$\varepsilon$ &  &  &  &  &  &  &  &  &  \\
\midrule
0.003  & 3.71 & 133.54 & 1271.23 & 4916.29 & 5.97 & 223.93 & 2300.03 & 9207.89 \\
0.01 & 1.40 & 48.68 & 466.47 & 1791.55     & 2.02 & 73.64 & 710.40 & 2893.43 \\
0.03 & 0.49 & 16.09 & 153.78 & 600.50      & 0.65 & 22.01 & 218.63 & 836.82 \\
0.1 & 0.14 & 3.16 & 31.75 & 126.10         & 0.18 & 5.80 & 61.25 & 229.85 \\
0.3 & 0.06 & 1.72 & 17.60 & 67.50          & 0.09 & 2.37 & 18.32 & 66.73 \\
\bottomrule
\end{tabular}
\caption{Average per-batch Sinkhorn time in seconds, with and without warmstarting for 32x32 ImageNet OTFM training.}
\label{tab:warmstart}
\end{table}

\subsection{Sinkhorn Speedup from PCA}\label{app:pca}
\Cref{tab:pca} presents the average wall-clock time of running Sinkhorn on ImageNet-64 with a batch size of $131072$ and $\varepsilon=0.1$. As can be seen, we can reduce dimension by a factor of almost 25, which reduces time by a factor of 10, while having no significant impact on the quality of generated images measured by FID. Moreover, the normalized entropy demonstrates that the coupling obtained from the reduced-dimensional cost matrix has the same sharpness as the original coupling, which is expected since PCA will mostly preserve the dot product cost, resulting in similar couplings.

\begin{table}[h]
    \renewcommand{\arraystretch}{1.3}
    \centering
    \begin{tabular}{ccccc}
        \toprule
         & $k=500$ & $k=1000$ & $k=3000$ & $k=12288$ (full dimension)\\
         \hline
        \textit{Sinkhorn time} & \textit{1.45s} & \textit{1.82s} & \textit{4.05s} & \textit{14.1s} \\
        \hline
        \hline
        FID@NFE=4 & $48.4$ & $48.1$ & $47.0$ & $47.3$\\
        \hline
        FID@NFE=8 & $24.7$ & $24.4$ & $24.0$ & $24.2$\\
        \hline
        FID@NFE=16 & $16.0$ & $15.8$ & $15.8$ & $15.8$\\
        \hline
        FID@Dopri5 (Adaptive) & $9.17$ & $9.33$ & $9.46$ & $9.51$\\
        \hline
        \hline
        Renormalized Entropy & $0.247$ & $0.239$ & $0.232$ & $0.236$\\
        \bottomrule
    \end{tabular}
    \vspace{3mm}
    \caption{Sinkhorn runtime per batch and FID for different solvers and different PCA dimension $k$. The model is trained on ImageNet-64 with OT batch size = $131072$ and $\varepsilon=0.1$. Note that the difference in FID for full $k$ compared to \Cref{tab:fid_main} is due to using a different random seed for training.}
    \label{tab:pca}
\end{table}

\subsection{Gaussian Transported with a Piecewise Affine Ground-Truth OT Map}\label{app:example_gaussian}
We present in Figure~\ref{fig:gaussian-example} examples of our piecewise affine OT map generation, corresponding to results presented more widely in Figures
~\ref{fig:piecewise} and ~\ref{fig:piecewise-final-app-entropy}.
\begin{figure}
\includegraphics[width=.49\textwidth]{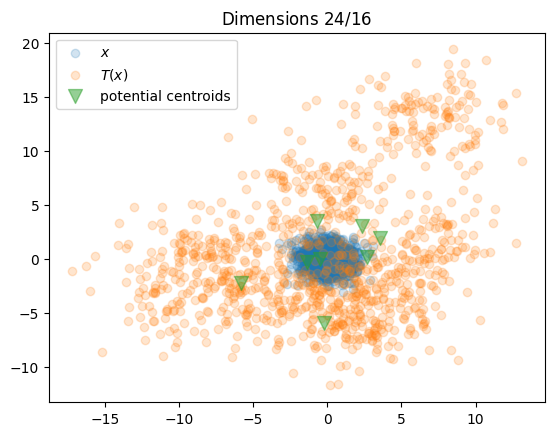}
\includegraphics[width=.49\textwidth]{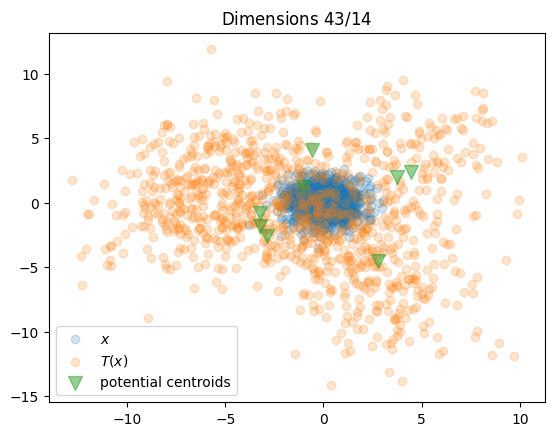}\\
\includegraphics[width=.49\textwidth]{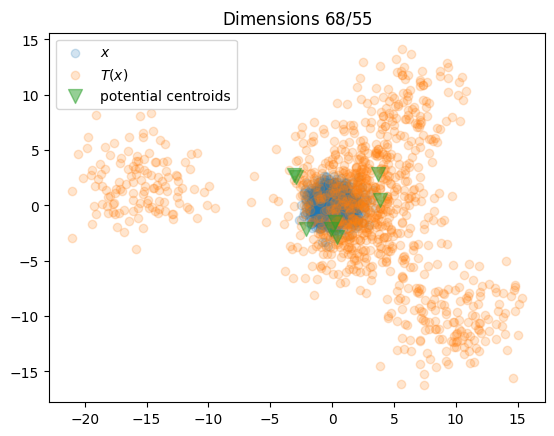}
\includegraphics[width=.49\textwidth]{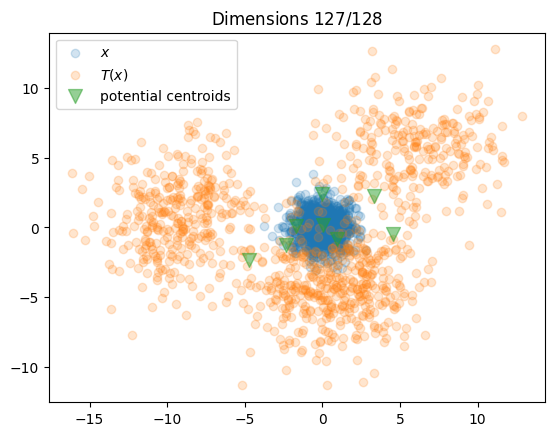}
\caption{Example of the maps generated in our piecewise affine benchmark task. In these plots $d=128$ and there are therefore $128/16=8$ quadratic potentials sampled around $0$. These 2D plots illustrate the action of the same $128$ dimensional map, pictured using 2D projections overs pairs chosen in $[1,\dots,128]$.}\label{fig:gaussian-example}
\end{figure}

\begin{figure}
\includegraphics[width=\linewidth]{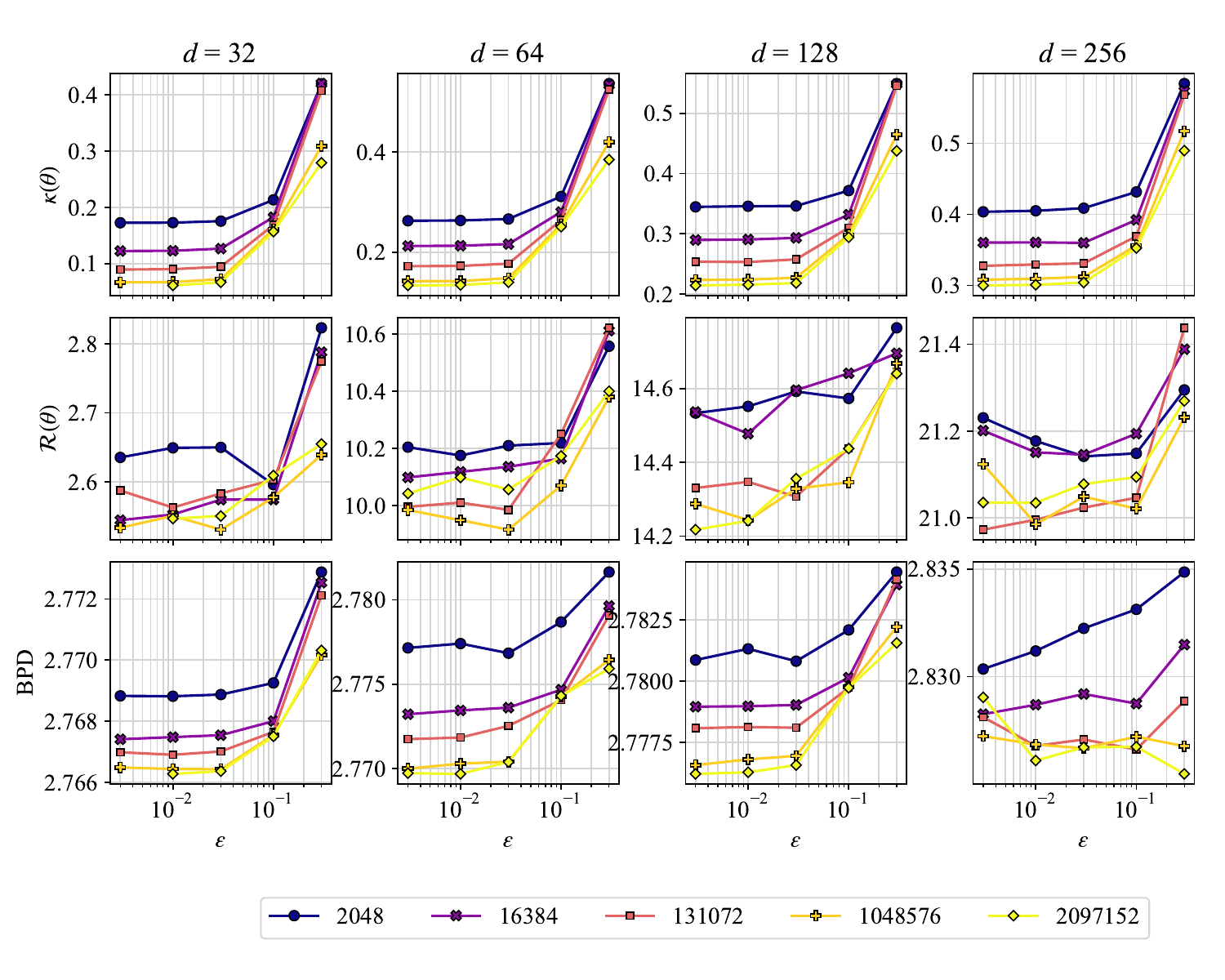} 
\includegraphics[width=\linewidth]{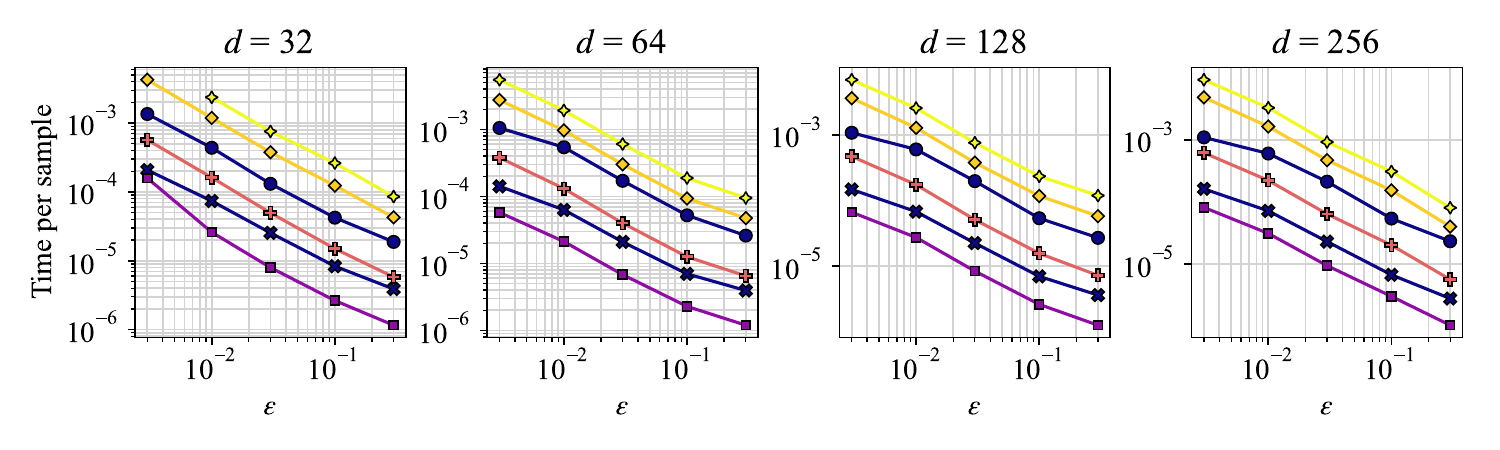} 
\caption{Plots corresponding to Figure~\ref{fig:piecewise} in main paper, on piecewise affine synthetic benchmark, using directly the relative \texttt{epsilon} parameter as the x-axis (log-scale), instead of re-normalized entropy.}\label{fig:piecewise-final-app-entropy}
\end{figure}

\subsection{\citeauthor{korotin2021neural} Benchmark Examples}\label{app:example_korotin}
The reader may find examples of the \citeauthor{korotin2021neural} benchmark in their paper, App. A.1, Figure 6.

\begin{figure}
\includegraphics[width=\linewidth]{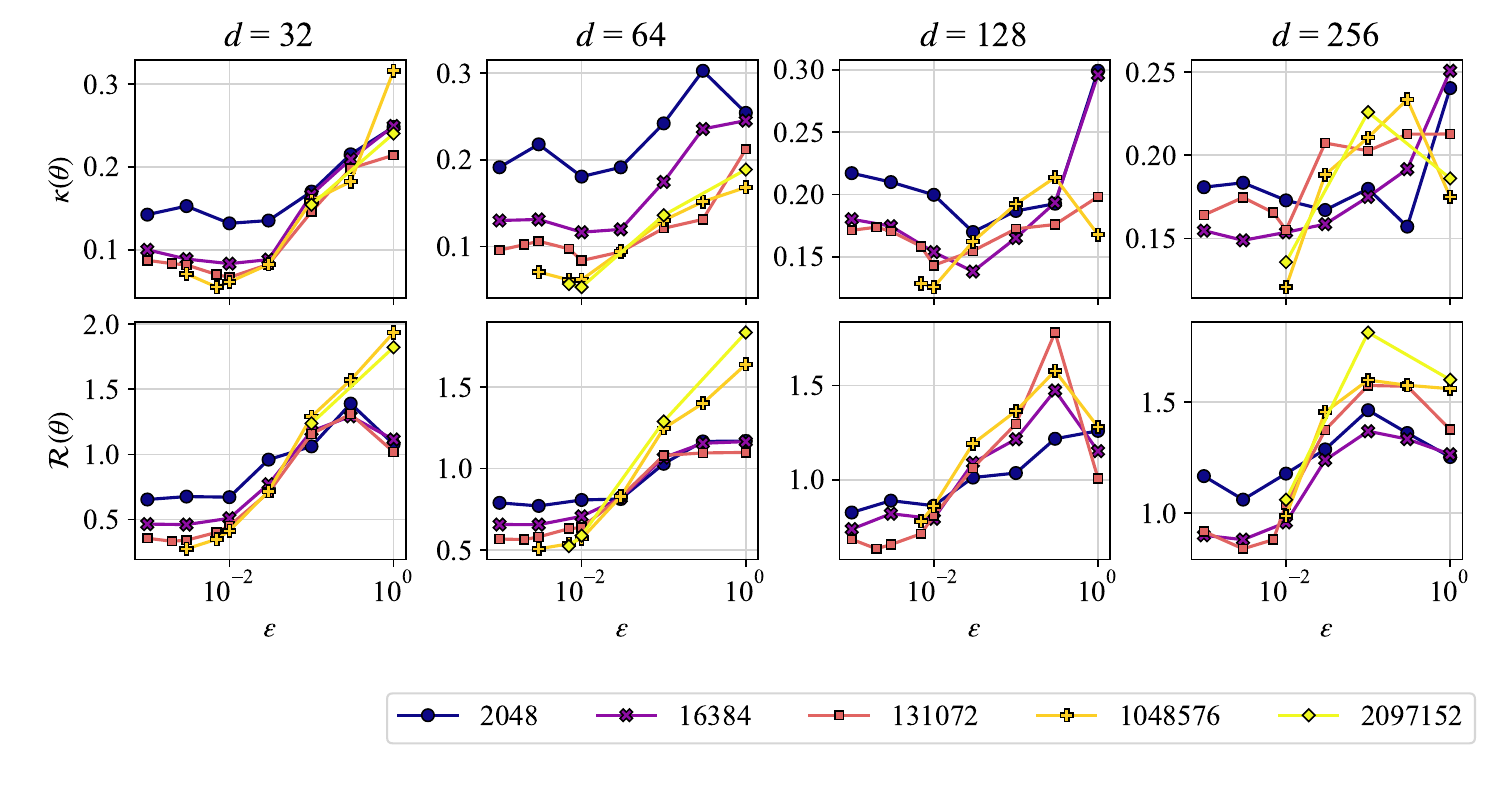} 
\includegraphics[width=\linewidth]{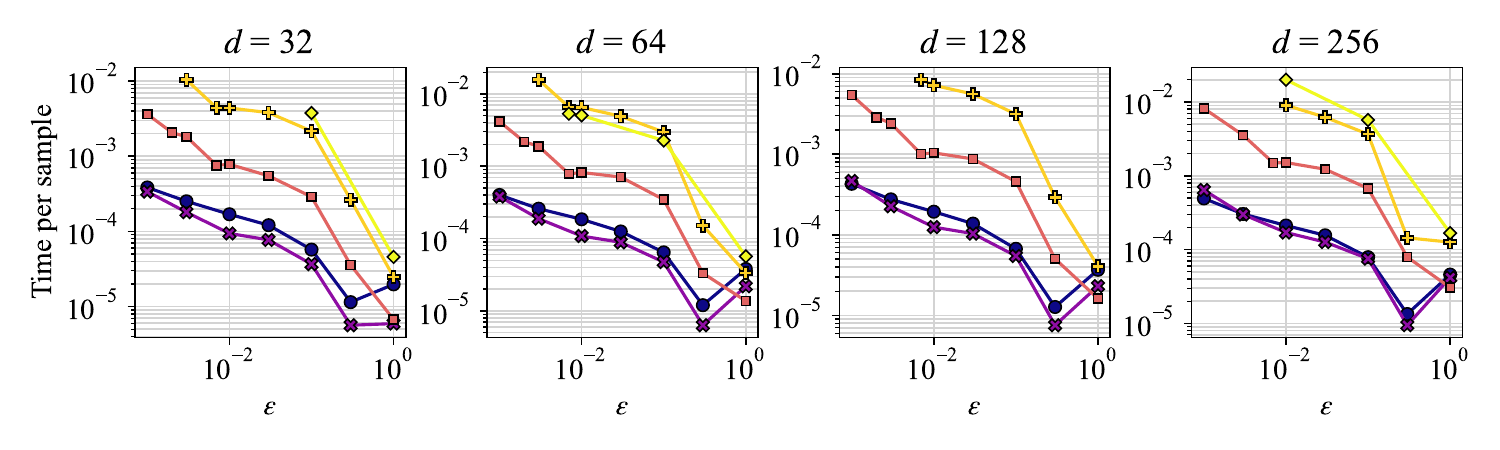} 
\caption{Plots for the \citeauthor{korotin2021neural} benchmark, shown initially in Figure~\ref{fig:korotin}, using the relative \texttt{epsilon} $\varepsilon$ parameter directly in the x-axis, in logarithmic scale.}\label{fig:korotin-final-app-epsilon}
\end{figure}

\subsection{CIFAR-10 Detailed Results}\label{app:cifar10}
We show generated images in Figure~\ref{fig:cifar10_grid}. We see general quantitative and qualitative improvements for larger OT batch size and smaller renormalized entropy. However, these improvements are not as significant as our observation for the more complex down-sampled ImageNet datasets in Appendices~\ref{app:imagenet32} and~\ref{app:imagenet64}, likely due to the fact that the dataset size is much smaller. We also plot BPD as a function of renomralized entropy for CIFAR-10, ImageNet-32, and ImageNet-64, in \Cref{fig:bpd}.

\begin{figure}[ht]
\centering
\includegraphics[width=.32\textwidth]{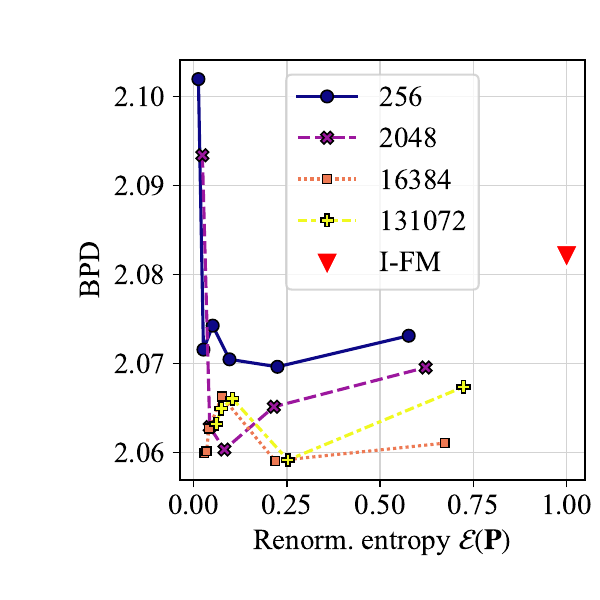}
\includegraphics[width=.32\textwidth]{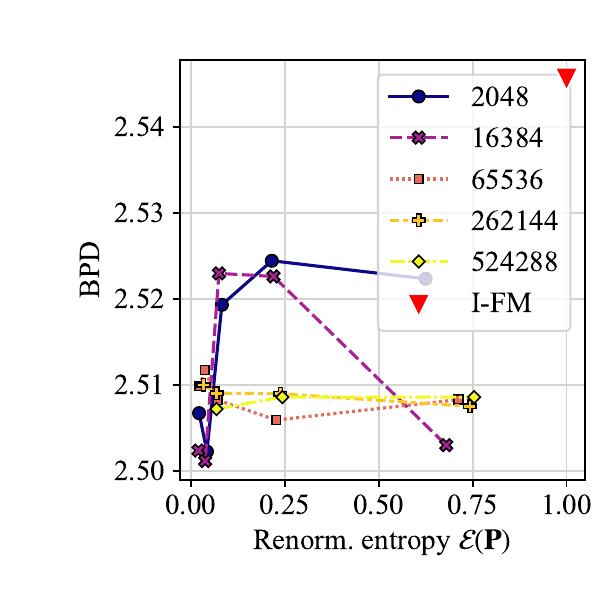}
\includegraphics[width=.32\textwidth]{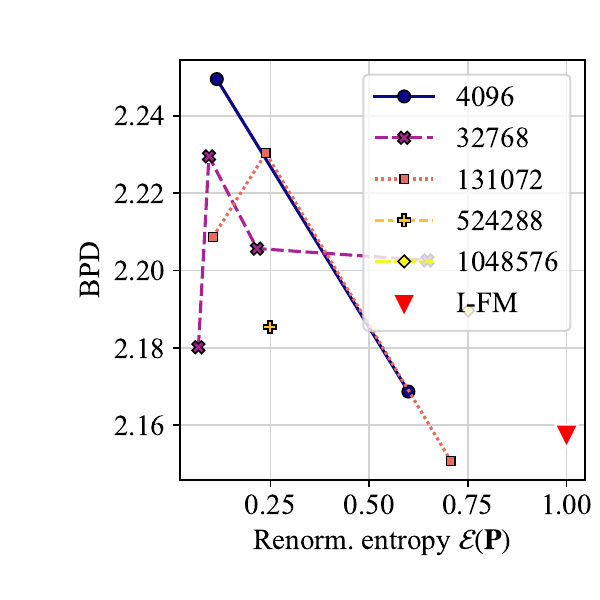}
\caption{BPD for \textbf{CIFAR-10} (Left),  \textbf{ImageNet-32} (Middle) and \textbf{ImageNet-64} (Right). The BPDs are computed using \texttt{Dopri5} integration, evaluated on 50 times steps, and computed using $8$ vectors for the Hutchinson trace estimator. As a consequence of its high number of function evaluations, the \texttt{Dopri5} solver relies less on straightness of the flows. Therefore, we do not observe a significant difference across batch sizes.}
    \label{fig:bpd}
\end{figure}

\begin{figure}[ht]
\centering
\input{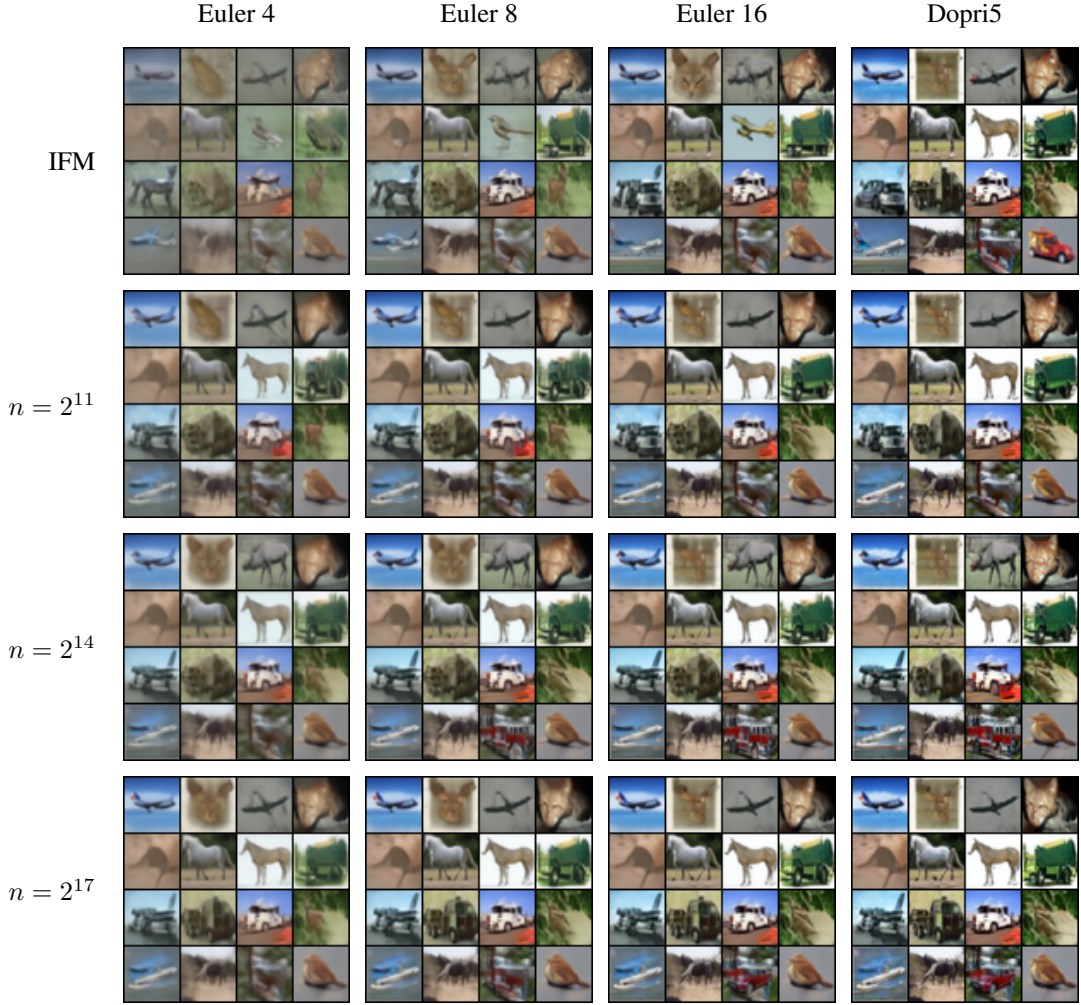}
\caption{Non-curated images generated from models trained on \textbf{CIFAR-10}. The number following \texttt{Euler} denotes NFE, while \texttt{Dopri5} uses an adaptive number of evaluations. $n$ denotes the total batch size for the Sinkhorn algorithm. We use OT-FM models trained with $\varepsilon=0.01$.}
\label{fig:cifar10_grid}
\end{figure}

\subsection{ImageNet-32 Detailed Results}\label{app:imagenet32}
Figure~\ref{fig:imagenet32_grid} shows generated images using I-FM and OT-FM with different batch sizes and different ODE solvers. As expected, the greatest improvements in the quality of images occur with smaller number of integration steps, which demonstrates the benefit of OT-FM for reducing inference cost.

\begin{figure}[ht]
\centering
\input{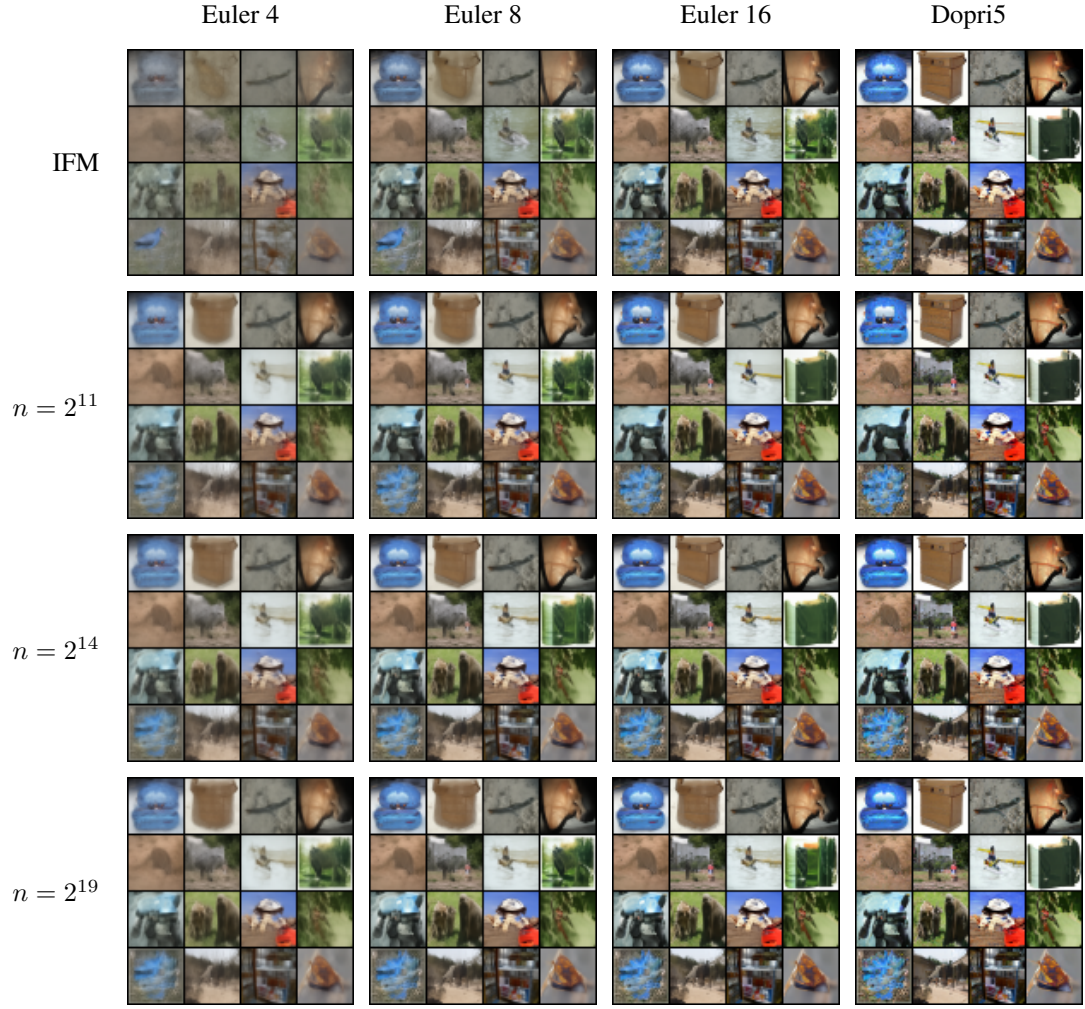}
\caption{Non-curated images generated from models trained on \textbf{ImageNet-32}. $n$ denotes the total batch size for the Sinkhorn algorithm. We use OT-FM models trained with $\varepsilon=0.1$.}
\label{fig:imagenet32_grid}
\end{figure}

\subsection{ImageNet-64 Detailed Results}\label{app:imagenet64}
We also perform experiments on the 64 $\times$ 64 downsampled ImageNet dataset, where we observe an even bigger gap between I-FM and OT-FM with large batch size both in terms of metrics (Figure~\ref{fig:imagenet64_results}) and in terms of qualitative results (Figure~\ref{fig:imagenet64_grid}). This observation implies that with a proper choice of entropy and batch size, OT-FM is a promising approach to reduce inference cost and generate higher quality high-resolution images.

\begin{figure}[ht]
\centering
\hspace*{-1cm}
\input{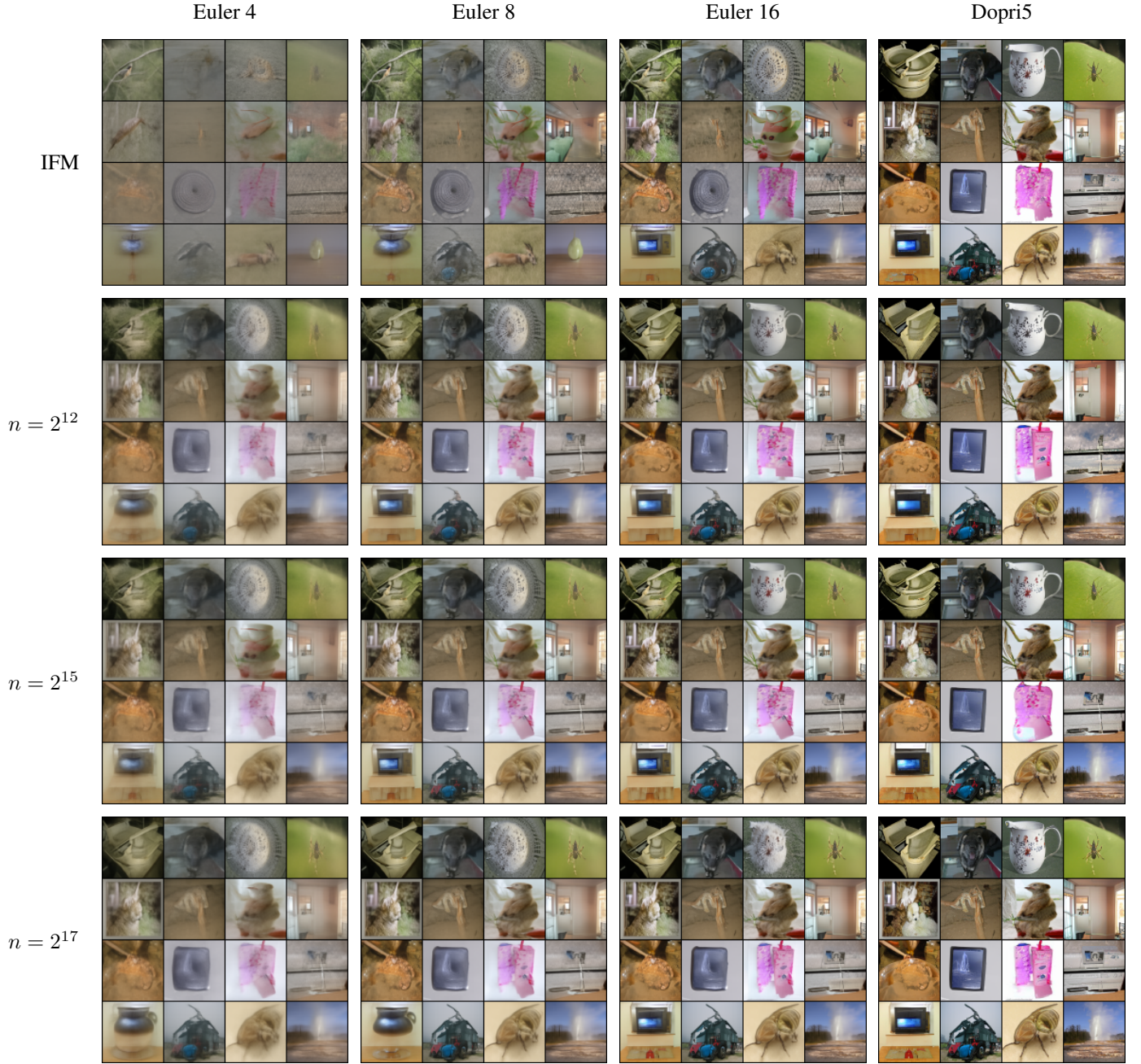}
\caption{Non-curated images generated from models trained on \textbf{ImageNet-64}. $n$ denotes the total batch size for the Sinkhorn algorithm. We use models trained with a varying trained with $\varepsilon=0.1$.}
\label{fig:imagenet64_grid}
\end{figure}

\end{document}